\newif\ifshortver
\shortverfalse

\ifshortver
\documentclass[conference]{IEEEtran}
\else
\documentclass[onecolumn]{IEEEtran}
\fi

\IEEEoverridecommandlockouts

\usepackage{cite}
\usepackage{amsmath,amssymb,amsfonts,amsthm,booktabs}
\usepackage{algpseudocode}
\usepackage{algorithm}
\usepackage{graphicx}
\usepackage{url}
\usepackage{textcomp}
\usepackage{xcolor}
\usepackage{caption}
\usepackage{enumitem}

\theoremstyle{plain}
\newtheorem{theorem}{Theorem}
\theoremstyle{definition}
\newtheorem{definition}[theorem]{Definition}
\theoremstyle{plain}
\newtheorem{lemma}[theorem]{Lemma}
\theoremstyle{plain}

\theoremstyle{plain}
\newtheorem{proposition}[theorem]{Proposition}
\theoremstyle{definition}

\def\BibTeX{{\rm B\kern-.05em{\sc i\kern-.025em b}\kern-.08em
    T\kern-.1667em\lower.7ex\hbox{E}\kern-.125emX}}
\DeclareMathOperator{\Var}{Var}
\begin{document}

\title{Communication-Efficient and Privacy-Adaptable Mechanism -- a Federated Learning \\
Scheme with Convergence Analysis
}

\ifshortver
\else
\author{\IEEEauthorblockN{Chun Hei Michael Shiu} \\
\IEEEauthorblockA{Department of Electrical and Computer Engineering, University of British Columbia \\
shiuchm@ece.ubc.ca} \\
\and
\IEEEauthorblockN{Chih Wei Ling} \\
\IEEEauthorblockA{Department of Computer Science, City University of Hong Kong\\
cwling6@um.cityu.edu.hk}
}
\fi

\maketitle

\begin{abstract}
    Federated learning enables multiple parties to jointly train learning models without sharing their own underlying data, offering a practical pathway to privacy-preserving collaboration under data-governance constraints. Continued study of federated learning is essential to address key challenges in it, including communication efficiency and privacy protection between parties. A recent line of work introduced a novel approach called the Communication-Efficient and Privacy-Adaptable Mechanism (CEPAM), which achieves both objectives simultaneously. CEPAM leverages the rejection-sampled universal quantizer (RSUQ), a randomized vector quantizer whose quantization error is equivalent to a prescribed noise, which can be tuned to customize privacy protection between parties. In this work, we theoretically analyze the privacy guarantees and convergence properties of CEPAM. Moreover, we assess CEPAM's utility performance through experimental evaluations, including convergence profiles compared with other baselines, and accuracy-privacy trade-offs between different parties.
\end{abstract}




\section{Introduction and Motivation}

Federated learning (FL) was first introduced in \cite{McMahan2016FL} as a framework where many decentralized clients collaboratively train a shared global model by sending only updates (e.g. gradients, model differences) to a central server, without transmitting data from their underlying databases. Such decentralized design\cite{Jeffrey2012distributed} exploits their computational resource to handle vast amounts of data held by each client. To make such framework practical, the authors in \cite{McMahan2016FL} propose the FederatedAveraging (\texttt{FedAvg}) approach, empirically showing that it can train deep neural networks efficiently. Unlike in classical centralized machine learning, communication efficiency and privacy protection are primary concerns\cite{Tian2019FLChallenges} in the study of federated learning, due to the distributed nature and the need to limit information exchange between parties. Traditionally, these challenges have been addressed separately, with compression or quantization schemes \cite{Konecn2016NIPSFederatedLS,lin2020DGCreducing,Corentin2017DLadapCompr,alistarh2017QSGD,shlezinger2020uveqfed} improving communication efficiency, and differential privacy (DP) mechanisms \cite{Dwork06DP,Dwork14Book} enhancing privacy protection.

To address communication efficiency and privacy protection simultaneously, a common strategy is to first apply a differential privacy (DP) mechanism (e.g. Gaussian\cite{Dwork14Book} or Laplace\cite{Dwork06DP} noise), and then quantize the perturbed updates\cite{girgis2021shuffled,agarwal2021skellam}. However, this DP-then-quantize pipeline is suboptimal: as the aggregated error from the privacy mechanism and quantization scheme is no longer exactly controlled, which harms the model accuracy. Also, the quantization error itself does not contribute to privacy protection against other curious clients. An alternative line of work uses randomized quantization such as subtractive dithering 
\ifshortver
\else 
\cite{roberts1962picture,schuchman1964dither, Limb1969visual, Jayant1972speech, sripad1977necessary},
\fi
to jointly achieve compression and privacy. However, methods based on subtractive dithering 
\ifshortver
\else
\cite{Kim2021FederatedLW,shahmiri2024communication,hasircioglu2023communication,yan2023layered,hegazy2024compression} 
\fi
typically rely on scalar quantization, which is less efficient than vector quantization. In addition, they often support only certain noise distributions such as Laplace or multivariate-$t$ distributions \cite{lang2023joint},\footnote{Although \cite{lang2023joint} proposes adding a low-power privacy-preserving noise before quantization, it does not specify how to design the necessary infinitesimal noise for local DP's Gaussian noise, a highly non-trivial task. In contrast, our CEPAM approach provides a clear specification for generating Gaussian noise under the central DP model.} rather than fully encompassing the Gaussian mechanism.

Another related line of work is the study of channel simulation \cite{bennett2002entanglement,harsha2010communication,sfrl_trans,havasi2019minimal,shah2022optimal,flamich2023greedy,liu2024universal,Li2024ChannelSim}, where one seeks to reproduce the effect of 
\ifshortver
\else
a noise channel with or without shared randomness. 
\fi
Based on the observation that Gaussian noise can be rewritten as a mixture of uniform distributions\cite{WalkerUniformP1999,CHOY2003exppower}, constructions in \cite{agustsson2020universally, hegazy2022randomized} combine randomized quantization and dithering to simulate any one-dimensional additive-noise channel, and more recent results \cite{ling2025RSUQ} extend these ideas to general multivariate distributions via randomized vector quantization. In particular, 
\ifshortver
\else 
the rejection-sampled universal quantizer (RSUQ) \cite{ling2025RSUQ} 
\fi
has been shown to exactly simulate any multivariate non-uniform continuous additive-noise channel with finite communication.

More recently, a joint mechanism 
\ifshortver
\else 
called the Communication-Efficient and Privacy-Adaptable Mechanism (CEPAM) 
\fi
has been proposed \cite{ling2025communication} for the FL setting. CEPAM is built upon RSUQ by exploiting its ability to reinterpret quantization distortion as an additive noise term with controllable variance, independent of the underlying model updates. In particular, by appropriately configuring RSUQ,\footnote{A variation of RSUQ -- layered rejection-sampled universal quantizer (LRSUQ) was used.} CEPAM can emulate different privacy mechanisms while simultaneously compressing the updates, thereby enabling a tunable trade-off between privacy protection and communication cost within a unified framework.

This work is a continuation of the study of CEPAM in the FL setting. In contrast to \cite{ling2025communication}, which studies the performance of CEPAM when compressing the model differences as updates, we consider a slightly modified setup in which CEPAM is applied to compress the stochastic gradients. Moreover,
\ifshortver
\else
whereas the earlier work reported empirical convergence behavior of CEPAM without a formal analysis, \fi
we provide a rigorous convergence analysis together with theoretical privacy guarantees for the proposed scheme.

We theoretically analyze the convergence behavior of CEPAM, and validate our convergence results for  \ifshortver
\else
both CEPAM-Gaussian and CEPAM-Laplace \fi 
through experimental evaluations.
Experimental results demonstrate that
\ifshortver\else
both CEPAM-Gaussian and CEPAM-Laplace achieve improvements of 0.8-1.1\% in test accuracy, respectively\fi, compared to several other commonly used baselines in the FL setup. In addition, we investigate the accuracy-privacy trade-offs among clients through experimental evaluations\ifshortver\else , and the observed behavior coincides with that obtained when compressing model differences in the previous study \cite{ling2025communication}. \fi

\ifshortver
\else
The rest of this paper is structured as follows. In Section \ref{sec: model}, we define the system model and review preliminaries. In Section \ref{sec: CEPAM}, we detail the proposed scheme CEPAM and highlight the modifications. In Section \ref{sec: analysis}, we provide theoretical analysis on the performance of CEPAM, and the corresponding experimental evaluations are given in \ref{sec: exp}. 
Finally, we conclude the paper in Section \ref{sec: conclusion}.
\fi

\section{System Model and Preliminaries} \label{sec: model}

In this section, we describe the system model and setups\ifshortver \else used in this paper\fi. We first present the FL framework in Section \ref{sec: fl}, followed by a detailed description of LRSUQ in Section \ref{sec: rsuq}. Section \ref{sec: dp} reviews the necessary background on differential privacy\ifshortver\else. 
Finally, we identify and summarize the desired goals under our framework in Section \ref{sec: problem setting}.
\fi

\subsection*{Notations}

Write $H(X)$ for the entropy and in bits.
Logarithms are to the base $2$.  
For $\mathcal{A},\mathcal{B}\subseteq\mathbb{R}^{n}$,
$\beta\in\mathbb{R}$, $\mathbf{G}\in\mathbb{R}^{n\times n}$, $\mathbf{x}\in\mathbb{R}^{n}$
write $\beta\mathcal{A}:=\{\beta\mathbf{z}:\,\mathbf{z}\in\mathcal{A}\}$,
$\mathbf{G}\mathcal{A}:=\{\mathbf{G}\mathbf{z}:\,\mathbf{z}\in\mathcal{A}\}$,
$\mathcal{A}+\mathbf{x}:=\{\mathbf{z}+\mathbf{x}:\,\mathbf{z}\in\mathcal{A}\}$,
$\mathcal{A}+\mathcal{B}=\{\mathbf{y}+\mathbf{z}:\,\mathbf{y}\in\mathcal{A},\,\mathbf{z}\in\mathcal{B}\}$
for the Minkowski sum, $\mathcal{A}-\mathcal{B}=\{\mathbf{y}-\mathbf{z}:\,\mathbf{y}\in\mathcal{A},\,\mathbf{z}\in\mathcal{B}\}$,
and $\mu(\mathcal{A})$ for the Lebesgue measure of $\mathcal{A}$.
Let $B^{n}:=\{\mathbf{x}\in\mathbb{R}^{n}:\,\Vert\mathbf{x}\Vert\le1\}$ be the unit $n$-ball.
For a function $f:\mathbb{R}^n \rightarrow \mathbb{R}$, its superlevel set is defined as $L_{u}^{+}(f):=\{\mathbf{x} \in \mathbf{R}^n:f(\mathbf{x}) \geq u\}$.

\subsection{Federated Learning} \label{sec: fl}

For the setup, we consider the FL framework 
(or \emph{federated optimization})
proposed in \cite{McMahan2016FL}.
More explicitly, $K$ clients (or devices) possess local dataset $\mathcal{D}^{(k)}$ where $k \in \{1, 2, ..., K\}  =: \mathcal{K}$, work jointly together to train a shared global model $\mathbf{W}$ with $m$ parameters through a central server.
The goal is to minimize the  objective function $F:\mathbb{R}^m \to \mathbb{R}$:
\begin{equation} \label{eq:FLOpt}
    \min_{\mathbf{W} \in \mathbb{R}^m} \Big\{F(\mathbf{W}):= \sum_{k \in \mathcal{K}} p_{k}F_{k}(\mathbf{W})\Big\},
\end{equation}
where $p_k$ is the weight of client $k$ such that $p_k \ge 0$ and $\sum_{k \in \mathcal{K}} p_{k} = 1$.
Suppose that the $k$-th local dataset contains $n_{k}$ training data: $\mathcal{D}^{(k)} = \{\xi_{k,1}, \dots, \xi_{k,n_{k}}\}$. The local objective function $F_k:\mathbb{R}^m \to \mathbb{R}$ is defined by
\begin{equation}
    F_{k}(\mathbf{W}) \equiv F_k(\mathbf{W},\mathcal{D}^{(k)}) := \frac{1}{n_k} \sum_{j=1}^{n_{k}} \ell(\mathbf{W};\xi_{k,j}),
\end{equation}
where $\ell(\cdot;\cdot)$ is an application-specified loss function.

Let $T$ denote the total number of iterations in FL and define the set of \emph{synchronization indices} $\mathcal{T}_T := \{0, \tau, 2\tau, \ldots, T\}$, i.e., the set of integer multiples of some positive integer $\tau \in \mathbb{Z}^+$ where $T \equiv 0 \pmod{\tau}$. 
We describe one FL round of the FedAvg \cite{McMahan2016FL} with some modifications to solve the optimization problem \eqref{eq:FLOpt} as follows. 
Let $\mathbf{W}_t$ denote the global parameter vector available on the  server at the time instance $t \in \mathcal{T}_T$. 
At the beginning of each FL round, the server broadcasts $\mathbf{W}_t$ to all clients.
Then, each client $k$ sets $\mathbf{W}_{t}^{k}=\mathbf{W}_{t}$ and computes the $\tau-1 \; (\ge 1)$ local parameter vectors by SGD:\footnote{When $T$ is fixed, the larger $\tau$ is, the fewer the communication rounds there are.}
\begin{equation} \label{eq:local_sgd}
    \mathbf{W}_{t+t'}^{k} \leftarrow \mathbf{W}_{t+t'-1}^{k} - \eta_{t+t'-1} \nabla F_{k}^{j_{t+t'-1}^{k}}(\mathbf{W}_{t+t'-1}^{k}),
\end{equation}
for $t'=1, \ldots, \tau-1,$ where $\eta_{t+t'-1}$ is the learning rate, $\nabla F_{k}^{j}(\mathbf{W}):=\nabla F_{k}(\mathbf{W};\xi_{k,j})$ is the gradient computed at a single sample of index $j$, and $j_{t}^{k}$ is the sample index chosen uniformly from the local data $\mathcal{D}^{(k)}$ of client $k$ at time $t$.\footnote{In this work, our focus is on analyzing the computation of a single stochastic gradient at each client during every time instance.
The FL convergence rates can potentially be enhanced by incorporating mini-batching technique \cite{stich2018local}, leaving the detailed analysis for future work.}
Finally, client $k$ computes $\mathbf{X}_{t+\tau-1}^{k} = \nabla F_{k}^{j_{t+\tau}^{k}}(\mathbf{W}_{t+\tau-1}^k)$.
Let $M$ be the maximum $\ell_2$-norm of all possible gradients for any given weight vector $\mathbf{W}$ and sampled dataset $\xi_k$, i.e., 
\ifshortver
\else
\begin{align*}
    M := \sup_{\mathbf{W} \in \mathbb{R}^m, \xi_k \in \mathcal D^{(k)}} \mathbb{E}\left[\left\| \nabla F_k(\mathbf{W}; \xi_{k,j})\right\|_2\right]. 
\end{align*} 
\fi
For simplicity, we assume that all $K$ clients participate in each FL round for all $k \in \mathcal K$. 
The server then aggregates $K$ local gradients $\{\mathbf{X}_{t+\tau-1}^{k}\}_{k\in \mathcal{K}}$, computes the new global parameter vector:
\begin{equation} \label{eq:global_without_quan}
    \mathbf{W}_{t+\tau} \leftarrow \mathbf{W}_{t} - \eta_{t+\tau-1}\sum_{k \in \mathcal{K}} p_{k}\mathbf{X}_{t+\tau-1}^{k},
\end{equation}
and broadcasts  $\mathbf{W}_{t+\tau}$ to all clients.

The dataset $\mathcal{D}^{(k)}$ inherently induces a distribution. 
By an abuse of notation, we also denote this induced distribution as $\mathcal{D}^{(k)}$.
Suppose the data in client $k$ is iid sampled from the induced distribution $\mathcal{D}^{(k)}$.
Thus, the overall distribution becomes a mixture of all local distributions: $\mathcal{D} = \sum_{k \in \mathcal{K}} p_k \mathcal{D}^{(k)}$. 
Previous works typically assume that the data is iid generated by or partitioned among the $K$ clients, i.e., for all $k \in \mathcal{K}$, $\mathcal{D}^{(k)} = \mathcal{D}$.
In contrast, we consider a scenario where the data is non-iid (or heterogeneous), implying that $F_k$ could potentially be an arbitrarily poor approximation to $F$.

\subsection{Rejection-Sampled Universal Quantizer} \label{sec: rsuq}

We briefly review  RSUQ \cite{ling2025RSUQ}, which is constructed based on the subtractive dithered  quantizer (SDQ) \cite{roberts1962picture,ziv1985universal,zamir1992universal}.

Given a non-singular generator matrix
$\mathbf{G}\in\mathbb{R}^{n\times n}$,  a \emph{lattice} is the set $\mathbf{G}\mathbb{Z}^{n}=\{\mathbf{G}\mathbf{j}:\,\mathbf{j}\in\mathbb{Z}^{n}\}$. 
A bounded set $\mathcal{P}\subseteq\mathbb{R}^{n}$ is called a \emph{basic
cell} of the lattice $\mathbf{G}\mathbb{Z}^{n}$ if $(\mathcal{P}+\mathbf{G}\mathbf{j})_{\mathbf{j}\in\mathbb{Z}^{n}}$
forms a partition of $\mathbb{R}^{n}$ \cite{conway2013sphere, zamir2014}.
Specifically, the Voronoi
cell $\mathcal{V}:=\{\mathbf{x}\in\mathbb{R}^{n}:\,\arg\min_{\mathbf{j}\in\mathbb{Z}^{n}}\Vert\mathbf{x}-\mathbf{G}\mathbf{j}\Vert=\mathbf{0}\}$ is a basic cell.
Given a basic cell $\mathcal{P}$, we can define a \emph{lattice quantizer} $Q_{\mathcal{P}}:\mathbb{R}^{n}\to\mathbf{G}\mathbb{Z}^{n}$ such that $Q_{\mathcal{P}}(\mathbf{x})=\mathbf{y}$
where $\mathbf{y}\in\mathbf{G}\mathbb{Z}^{n}$ is the unique lattice
point that satisfies $\mathbf{x}\in-\mathcal{P}+\mathbf{y}$. 
The resulting quantization error $\mathbf{z}:=Q_{\mathcal{P}}(\mathbf{x})-\mathbf{x}$ depends deterministically on the input $\mathbf{x}$ and is approximately uniformly distributed over
 the basic cell of the lattice quantizer under some regularity assumptions. 
Therefore, it is often combined with probabilistic methods such as random dithering to construct SDQ  \cite{ziv1985universal,zamir1992universal}.

\begin{definition}
Given a basic cell $\mathcal{P}$ and a random dither $\mathbf{V}\sim\mathrm{Unif}(\mathcal{P})$, a \emph{subtractive dithered
quantizer} (SDQ) $Q_{\mathcal{P}}^{SDQ}:\mathbb{R}^{n}\times\mathcal{P}\to\mathbb{R}^{n}$ for an input $\mathbf{x} \in \mathbb{R}^n$ is given by 
\ifshortver
\else 
$Q_{\mathcal{P}}^{SDQ}(\mathbf{x},\mathbf{v})=Q_{\mathcal{P}}(\mathbf{x}-\mathbf{v})+\mathbf{v}$, where $Q_{\mathcal{P}}$ is the lattice quantizer. 
\fi
\end{definition}

It is well-known that the resulting quantization error of SDQ is uniformly distributed over the basic cell of the quantizer and is statistically independent of the input signal \cite{schuchman1964dither,zamir1992universal,gray1993dithered,kirac1996results,zamir1996lattice}.
However, it may be desirable to have the quantization error follow a uniform distribution over an arbitrary set, rather than being distributed uniformly over a basic cell. 
RSUQ is a randomized quantizer where the quantization error is uniformly distributed over a set $\mathcal{A}$, a subset of a basic cell. This quantization scheme is based on applying rejection sampling on top of SDQ.
Intuitively, we keep generating new dither signals until the quantization error falls in $\mathcal{A}$.\footnote{ It is easy to see that the acceptance probability is $\mu(\mathcal{A})/\mu(\mathcal{P})$.}

\begin{definition} \cite[Definition 3]{ling2025RSUQ}
\label{def:rej_samp_quant} Given a basic cell $\mathcal{P}$ of the
lattice $\mathbf{G}\mathbb{Z}^{n}$, a subset $\mathcal{A}\subseteq\mathcal{P}$, and a sequence $S=(\mathbf{V}_{i})_{i\in\mathbb{N}^{+}}$, $\mathbf{V}_{1},\mathbf{V}_{2},\ldots\stackrel{iid}{\sim}\mathrm{Unif}(\mathcal{P})$
are i.i.d. dithers, 
the \emph{rejection-sampled universal quantizer} (RSUQ) $Q_{\mathcal{A},\mathcal{P}}:\mathbb{R}^{n}\times \prod_{i \in \mathbb{N}^{+}}\mathcal{P}_i\to\mathbb{R}^{n}$ for $\mathcal{A}$
against $\mathcal{P}$ is given by 
\ifshortver
\else
\begin{equation}
Q_{\mathcal{A},\mathcal{P}}(\mathbf{x},(\mathbf{v}_{i})_{i}):=Q_{\mathcal{P}}(\mathbf{x}-\mathbf{v}_{h})+\mathbf{v}_{h},\label{eq:QAP_def}
\end{equation}
where
\begin{equation}
h:=\min\big\{ i:\,Q_{\mathcal{P}}(\mathbf{x}-\mathbf{v}_{i})+\mathbf{v}_{i}-\mathbf{x}\in\mathcal{A}\big\},\label{eq:kstar}
\end{equation}
and $Q_{\mathcal{P}}$ is the lattice quantizer for basic cell
$\mathcal{P}$.
\fi
\end{definition}

Note that when $\mathcal{A}=\mathcal{P}$, SDQ is a special case of RSUQ. 
RSUQ can be generalized to simulate an additive noise channel with noise following a continuous  distribution by using layered construction as in \cite{hegazy2022randomized,wilson2000layered}.
Consider a pdf $f:\mathbb{R}^{n}\to[0,\infty)$ and write its superlevel set as
\ifshortver
\else
\[
L_{u}^{+}(f)=\{\mathbf{z}\in\mathbb{R}^{n}:\,f(\mathbf{z})\ge u\}.
\]
\fi
Let $f_{U}(u):=\mu(L_{u}^{+}(f))$ for $u>0$, which is also a pdf. 
If we generate $U\sim f_{U}$, and then $\mathbf{Z}|\{U=u\}\sim\mathrm{Unif}(L_{u}^{+}(f))$,
and $\mathbf{Z}\sim f$ by the fundamental theorem of simulation \cite{robert2004monte}.

\begin{definition} \cite[Definition 4]{ling2025RSUQ}
\label{def:rej_samp_quant_layer} Given a basic cell $\mathcal{P}$
of the lattice $\mathbf{G}\mathbb{Z}^{n}$, a probability density
function  $f:\mathbb{R}^{n}\to[0,\infty)$ where $L_{u}^{+}(f)$ is
always bounded for $u>0$, and $\beta:(0,\infty)\to[0,\infty)$ satisfying
$L_{u}^{+}(f)\subseteq\beta(u)\mathcal{P}$ for $u>0$, and a random pair $S=(U,(\mathbf{V}_{i})_{i\in\mathbb{N}^{+}})$ where the latent variable 
$U\sim f_{U}$ with $f_{U}(u):=\mu(L_{u}^{+}(f))$, and $\mathbf{V}_{1},\mathbf{V}_{2},\ldots\stackrel{iid}{\sim}\mathrm{Unif}(\mathcal{P})$
is a sequence of i.i.d. dither signals, the \emph{layered rejection-sampled universal quantizer} (LRSUQ) $Q_{f,\mathcal{P}}:\mathbb{R}^{n}\times \mathbb{R} \times \prod_{i \in \mathbb{N}^{+}}\mathcal{P}_i\to\mathbb{R}^{n}$ for $f$ against $\mathcal{P}$
is given by 
\ifshortver
\else
\begin{equation}
Q_{f,\mathcal{P}}(\mathbf{x},u,(\mathbf{v}_{i})_{i}):=\beta(u)\cdot\big(Q_{\mathcal{P}}(\mathbf{x}/\beta(u)-\mathbf{v}_{h})+\mathbf{v}_{h}\big),
\end{equation}
where
\begin{equation}
h:=\min\big\{ i:\,\beta(u)\cdot\big(Q_{\mathcal{P}}(\mathbf{x}/\beta(u)-\mathbf{v}_{i})+\mathbf{v}_{i}\big)-\mathbf{x}\in L_{u}^{+}(f)\big\},
\end{equation}
and $Q_{\mathcal{P}}$ is the lattice quantizer for basic cell
$\mathcal{P}$.
\fi
\end{definition}

It can be shown that LRSUQ indeed gives the desired error distribution.

\begin{proposition} \label{prop:LRSUQ_error} 
\cite[Proposition 4]{ling2025RSUQ}
For any random input $\mathbf{X}$, the quantization error of LRSUQ $Q_{f,\mathcal{P}}$
defined by $\mathbf{Z} := Q_{f,\mathcal{P}}(\mathbf{X},U,(\mathbf{V}_{i})_{i}) - \mathbf{X}$, 
follows the pdf $f$, independent of $\mathbf{X}$.
\end{proposition}

\subsection{Differential Privacy} \label{sec: dp}

The original FL framework \cite{McMahan2016FL} provides a certain level of privacy since clients do not directly transmit their private data to the  server.
However, a significant amount of information can still be inferred from the shared data, e.g., model parameters induced by gradient descent, by potential eavesdroppers within the FL network 
\ifshortver
\else
\cite{Zhu2019DeepLeak,zhao2020idlg,huang2021GradInvAtt}. 
\fi
Thus, a privacy mechanism, such as DP, is necessary to be equipped on the  system to protect the shared information.

In this section, we briefly review a notion of (central) DP  \cite{Dwork06DP,Dwork14Book}.\footnote{Herein, DP refers to \emph{central} DP.}
In DP, clients place trust in the server (or data curator) responsible for collecting and holding their individual data in a database $X \in \mathcal{X}$, where $\mathcal{X}$ denotes the collection of databases. 
The server then introduces privacy-preserving noise to the original datasets or query results through a randomized mechanism $\mathcal{F}$, producing an output $Y = \mathcal{F}(X) \in \mathcal{Y}$, where $\mathcal{Y}$ denotes the set of possible outputs, before sharing them with untrusted data analysts.
While this model requires a higher level of trust compared to the local model, it enables the design of significantly more accurate algorithms. 
Two databases $X$ and $X'$ are considered \emph{adjacent} if they differ in only one entry. 
 More generally, we can define a symmetric adjacent relation $\mathcal{R} \subseteq \mathcal{X}^2$ and say that $X$ and $X'$ are \emph{adjacent} databases if $(X, X') \in \mathcal{R}$.
Here, we review the definition of $(\epsilon, \delta)$-differentially private, initially introduced by Dwork et al. \cite{Dwork06DP}.

\begin{definition}[$(\epsilon, \delta)$-differential privacy \cite{Dwork06DP}] \label{def:edDP} 
    A randomized mechanism $\mathcal{F}:\mathcal{X} \rightarrow \mathcal{Y}$ with the associated conditional distribution $P_{Y|X}$ of $Y=\mathcal{F}(X)$ is $(\epsilon, \delta)$-\emph{differentially private} ($(\epsilon,\delta)$-DP) if for all $\mathcal{S} \subseteq \mathcal{Y}$ and for all $(X, X') \in \mathcal{R}$,
    \[
    \mathbb{P}(\mathcal{F}(X) \in \mathcal{S}) \leq e^{\epsilon}\mathbb{P}(\mathcal{F}(X') \in \mathcal{S})+\delta.
    \] 
When $\delta = 0$, we say that $\mathcal{F}$ is $\epsilon$-\emph{differentially private} ($\epsilon$-DP).
\end{definition}

In this work, we use the principle of privacy amplification by subsampling \cite{Balle2018PrivAmpl}, whereby the privacy guarantees of a differentially private mechanism are amplified by applying it to a random subsample of the dataset.
\ifshortver
\else 
The problem of privacy amplification can be stated as follows. 
Let $\mathcal{F}:\mathcal{X} \rightarrow \mathcal{Y}$ be a privacy mechanism with privacy profile $\delta_{\mathcal{F}}$ with respect to the adjacent relation defined on $\mathcal{X}$, and let $s:\mathcal{W} \rightarrow \mathcal{X}$ be a subsampling mechanism.
Consider the subsampled mechanism $\mathcal{F}^{s}:\mathcal{W} \rightarrow \mathcal{Y}$ given by $\mathcal{F}^{s}(X) :=\mathcal{F}(s(X))$. 
The goal is to relate the privacy profile of $\mathcal{F}$ and $\mathcal{F}^s$. 
For a detailed exposition, we refer to \cite{Balle2018PrivAmpl}.
\fi

\subsection{Problem Setting} \label{sec: problem setting}

\textbf{Threat model}:
We adopt trusted aggregator model, i.e., the server is trusted.
Moreover, we assume that there are separate and independent sources of shared randomness between each client and the server to perform quantization. 
Clients engaged in the FL framework are assumed to be honest yet inquisitive\ifshortver
\else
, meaning they comply with the protocol but may attempt to deduce sensitive client information from the average updates received by the server. 
Privacy requirements are:
\begin{itemize}[left=0pt, itemindent=0pt, labelsep=0.5em, labelwidth=1.2em]
    \item Protecting the privacy of each client's local dataset from other clients, as the updated local gradients between rounds may inadvertently disclose sensitive information.
    \item Preventing privacy leaks from the final trained model upon completion of training, as it too may inadvertently reveal sensitive information. 
    This ensures the relevance of our solution in scenarios where clients are trusted, and the final trained model may be publicly released to third parties.
\end{itemize}
\fi

We intend to construct a mechanism that able to jointly handle privacy requirements  and compression demands (for lossless uplink channels with limited bandwidth) in a single local gradient update at each client within the FL framework.
Since the distribution of the model parameters and/or the induced gradient is often unknown to the clients, our focus lies in creating a universal mechanism applicable to any random source. Furthermore, we are exploring privacy mechanisms that introduce noise customized to various noise distributions. 
While Laplace mechanism provides pure DP protection, 
Gaussian mechanism only provides approximate DP protection with a small failure probability \cite{Dwork06DP,Dwork14Book}. 
However, it is well-known that Gaussian mechanism support tractability of the privacy budget in mean estimation \cite{Dong2022GDP,Mironov2017RDP}, an important subroutine in FL.
Such schemes can be formulated as mappings from the local gradient $\mathbf{X}_{t}^{k}=\nabla F_{k}^{j_{t}^{k}}(\mathbf{W}_{t}^k) \in \mathbb{R}^m$ at client $k$ at the time instance $t$ to the estimated gradient $\hat{\mathbf{X}}_{t}^{k} \in \mathbb{R}^m$ at the server, \ifshortver
\else 
aimed to achieve the following desired properties:
\begin{enumerate}[left=0pt, labelwidth=2.5em, labelsep=0.5em, align=left, itemindent=0pt]
    \item Privacy requirement : The perturbed query function generated by the privacy mechanism must adhere to $(\epsilon,\delta)$-DP (or $\epsilon$-DP). 
    \ifshortver
    \else
    For instance, ensuring that the mapping of the average of local gradients across clients $\frac{1}{K}\sum_{k \in \mathcal{K}}\mathbf{X}_{t}^k$ to the average of estimated gradients $\frac{1}{K}\sum_{k \in \mathcal{K}}\hat{\mathbf{X}}_{t}^k$ at the server satisfies $(\epsilon,\delta)$-DP. 
    \fi
    \item  Compression/communication-efficiency : The estimation $\hat{\mathbf{X}}_{t+\tau}^{k}$ from client $k$ to the server should be represented by finite bits per sample.
    \item Universal source : the scheme should operate reliably irrespective of the distribution of $\mathbf{X}_t^k$ and without prior knowledge of it. 
    \item Adaptable noise : The noise $Z$ in the privacy mechanism is customizable according to the required accuracy level and privacy protection.
\end{enumerate}
\fi

\section{CEPAM for FL} \label{sec: CEPAM}

\ifshortver
\else
In this section, 
we recap CEPAM (with some modifications), which was initially proposed in \cite{ling2025communication}. 
\fi
CEPAM utilizes a randomized quantizer to provide privacy enhancement in FL setting. Specifically, CEPAM modifies  schemes in \cite{shlezinger2020uveqfed,hasircioglu2023communication} by replacing universal quantization \cite{ziv1985universal,zamir1992universal} with LRSUQ  as outlined in Section \ref{sec: rsuq}. 
\ifshortver
\else
Moreover, CEPAM generalizes schemes in 
\ifshortver
\cite{hasircioglu2023communication,hegazy2022randomized},
\else
\cite{hasircioglu2023communication,hegazy2022randomized,yan2023layered},
\fi
as scalar quantization is a special case of vector quantization when the dimension is one. 
\fi
\ifshortver
\else
Using LRSUQ, as discussed in Section \ref{sec: rsuq} and \cite{wilson2000layered, hegazy2022randomized}, CEPAM is capable of addressing compression and privacy requirements simultaneously.
\fi
In particular, each client $k\in \mathcal K$ computes the stochastic gradient $\mathbf{X}_{t+\tau-1}^k = \nabla F_k^{j_{t+\tau-1}^k} (\mathbf{W}_{t+\tau-1}^k)$ after performing $\tau-1$ local SGD steps, then applies norm clipping to get $\tilde{\mathbf{X}}_{t+\tau-1}^k$, and applies LRSUQ to the clipped stochastic gradients at the end of each client side FL round (which is the iteration $t+\tau-1$). These quantized updates are then partitioned into vectors of suitable lengths as a set of messages $\{(H_j^k,\mathbf{M}_j^k)\}_{j \in \mathcal N}$ to the server. The server receives the messages from each client and adds dithers according to the shared randomness with each client. Next, the server decodes and collects the messages to recover $\hat{\mathbf{X}}_t^k$ as the estimated update for each client. Subsequently, the server averages the estimated gradients from all clients as $\frac{1}{K}\sum_{k \in \mathcal K} \hat{\mathbf{X}}_t^k$, computes the new global model $\mathbf{W}_{t+\tau}$ and broadcasts it to all clients for the next FL round, or output the global model if it is the last iteration $T$. 
\ifshortver
\else
A detailed pseudocode description of CEPAM is given in Algorithm \ref{alg:CEPAM} and a flow diagram of CEPAM is given in Figure \ref{fig: flow diagram}.
\fi

\ifshortver
\else
\begin{algorithm}[ht]
\caption{\textsc{CEPAM}} \label{alg:CEPAM}
\begin{algorithmic}[1]
\State \textbf{Inputs:} 
Number of total iterations $T$, number of local iterations $\tau$, number of clients $K$, local datasets $\{\mathcal{D}^{(k)}\}_{k \in \mathcal{K}}$, loss function $\ell(\cdot,\cdot)$, clipping threshold $\gamma>0$
\State \textbf{Output:} Global optimized model $\mathbf{W}_T$
\State \textbf{Initialization:} Client $k$ and the server agree on privacy budget $\epsilon>0$ and privacy relaxation $\delta$ for $(\epsilon,\delta)$-DP (or privacy budget $\epsilon>0$ for $\epsilon$-DP), shared seed $s_k$, lattice dimension $n$, generator matrix $\mathbf{G}$, basic cell $\mathcal{P}$ of $\mathbf{G}\mathbb{Z}^n$,
privacy-preserving noise $\mathbf{Z}\sim f$ with noise variance $\mathrm{Var}(f)>0$, latent variable $U \sim g(u):=\mu(L_{u}^{+}(f))$,  function $\beta:(0,\infty)\to[0,\infty)$ satisfying
$L_{u}^{+}(f)\subseteq\beta(u)\mathcal{P}$ for $u>0$, initial model parameter vector $\mathbf{W}_0$ 
\State \textbf{Protocol at client $k$:}
\For{$t+1 \notin \mathcal{T}_T$}
\State Receive $\mathbf{W}_{t}$ from the server or use $\mathbf{W}_0$ if $t=0$
\State Set $\mathbf{W}_{t}^k  \leftarrow \mathbf{W}_{t}$
\For{$t'=1$ \textbf{to} $\tau-1$}
\State Compute $\mathbf{W}_{t+t'}^{k} \leftarrow \mathbf{W}_{t+t'-1}^{k} - \eta_{t+t'-1} \nabla F_{k}^{j_{t+t'-1}^{k}}(\mathbf{W}_{t+t'-1}^{k})$
\EndFor
\State Compute $\mathbf{X}_{t+\tau-1}^{k} \leftarrow \nabla F_{k}^{j_{t+\tau}^{k}}(\mathbf{W}_{t+\tau-1}^k)$
\State Compute $\tilde{\mathbf{X}}_{t+\tau-1}^{k}\leftarrow\mathbf{X}_{t+\tau-1}^{k}/\max\{1,\Vert\mathbf{X}_{t+\tau-1}^{k}\Vert_2/\gamma\}$ 
\Comment{Perform norm clipping}
\State Run subroutine $\mathrm{ENCODE}(\tilde{\mathbf{X}}_{t+\tau-1}^{k}, s_k, \mathbf{G}, \mathcal{P}, f, \beta(u), N)$ \Comment{See Algorithm \ref{alg:Encode}}
\State Send $\{(H_{t+\tau-1,j}^k,\mathbf{M}_{t+\tau-1,j}^k)\}_{j \in \mathcal{N}}$ to server, using $N\cdot\left(H(\mathrm{Geom}(p(U)\;|\;U)+H(\lceil \log |\mathcal{M}(U)|\rceil\;|\;U)\right)$ bits  
\EndFor 
\State \textbf{Protocol at the server:}
\For{$t+\tau \in \mathcal{T}_T$}
\State Receive $\{H_{t+\tau-1,j}^k, \mathbf{M}_{t+\tau-1,j}^k\}_{j \in \mathcal{N}}$ from clients
\For{$k \in \mathcal{K}$}
\State Run subroutine $\mathrm{DECODE}(\{(H_{t+\tau-1,j}^k,\mathbf{M}_{t+\tau-1,j}^k)\}_{j \in \mathcal{N}}, s_k, \mathbf{G}, \mathcal{P}, f, \beta(u), N)$ \Comment{See Algorithm \ref{alg:Decode}}
\EndFor
\State Compute  
$\hat{\mathbf{W}}_{t+\tau} \leftarrow \mathbf{W}_{t} - \eta_{t+\tau-1}  \sum_{k \in \mathcal{K}} p_{k}\hat{\mathbf{X}}_{t+\tau-1}^{k}$
\State Set $\mathbf{W}_{t+\tau} \leftarrow \hat{\mathbf{W}}_{t+\tau}$ and broadcast $\mathbf{W}_{t+\tau}$ to all clients, or output $\mathbf{W}_T$ if $t+\tau=T$
\EndFor
\end{algorithmic}
\end{algorithm}

\begin{figure}
    \centering
    \includegraphics[width=0.8\linewidth]{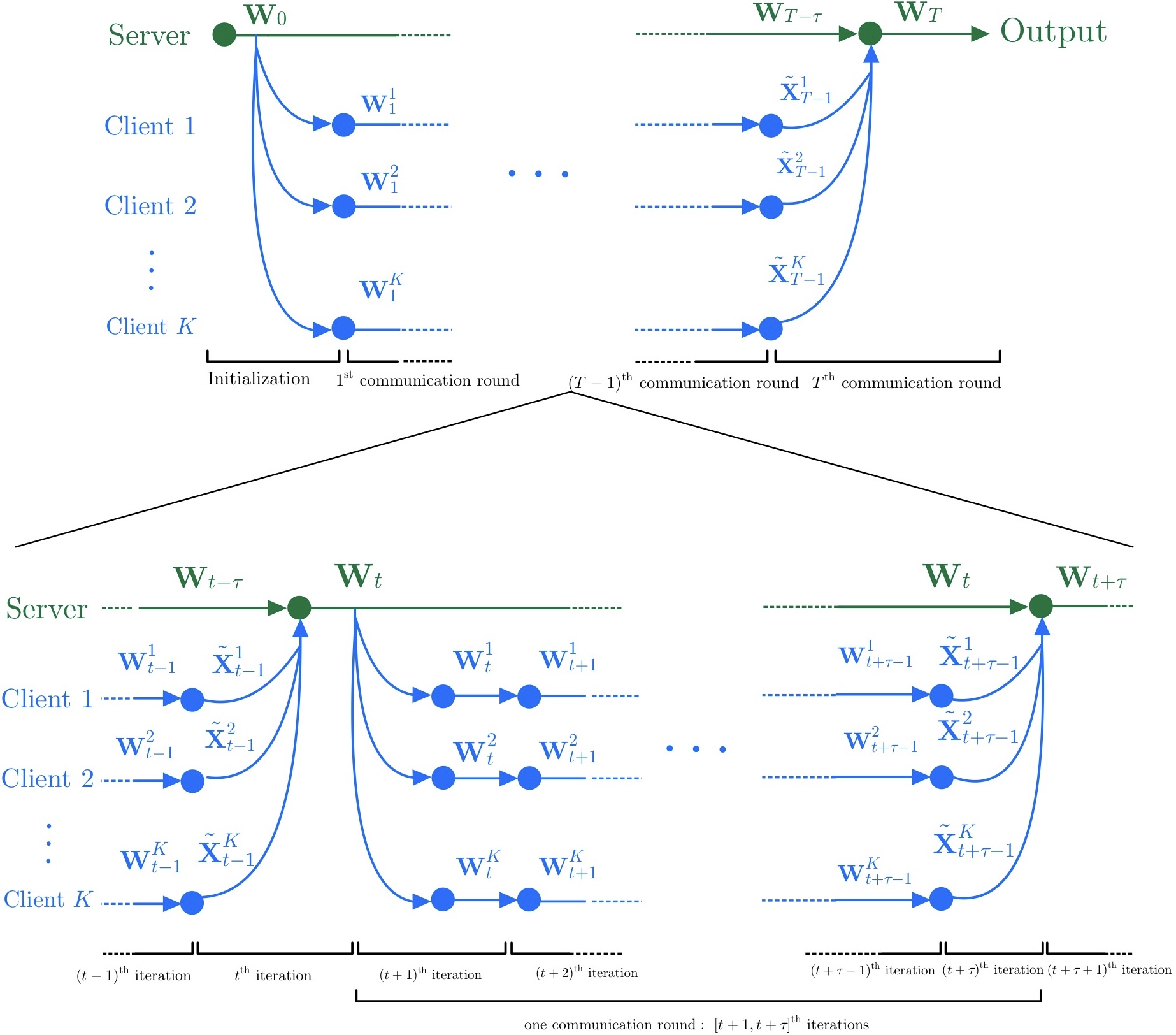}
    \caption{Schematic of CEPAM in the FL framework. The upper part demonstrates initialization and aggregation of model updates between server and clients across global communication rounds. 
    The lower part details one global communication round, illustrating the sequence of local updates over local iterations at clients and server-side aggregation at the end of one FL communication round.
    At the beginning of each communication round,  client $k$ receives the global parameter vector $\mathbf{W}_{t}$ from the server and set $\mathbf{W}_{t}^{k}=\mathbf{W}_t$. 
    Next, client $k$ performs $\tau -1$ steps of local SGD,  clips the gradient $\nabla F_{k}^{j_{t+\tau}^{k}}(\mathbf{W}_{t+\tau-1}^k)$, and encodes $\tilde{\mathbf{X}}_{t+\tau-1}^{k}$. 
    The server then aggregates $K$ estimated local gradients $\hat{\mathbf{X}}_{t+\tau-1}^{k}$ and perform one step of global SGD according to $\mathbf{W}_{t+\tau} = \mathbf{W}_{t}-\eta_{t+\tau-1}\sum_{k \in \mathcal{K}}p_k\hat{\mathbf{X}}_{t+\tau-1}^{k}$.}
    \label{fig: flow diagram}
    \vspace{-3pt}
\end{figure}
\fi

\ifshortver
\else
Besides the local trainings that are carried out at each client, CEPAM can be understood as a three-stage processing: an initialization stage, an encoding stage at each client, and a decoding stage at the server. Before giving the details of each stage, we briefly introduce them. The initialization stage involves both the clients and the server before the start of the FL training, including agreement on the parameters for privacy, randomized quantizers, and the initial model. The encoding stage is independently performed by each client, where the clipped gradient  is encoded into a set of messages. The decoding stage is carried out by the server, which decodes the set of messages from each client and aggregates them to recover an estimated gradient. We assume that all clients execute the same encoding function in every FL global epoch, ensuring consistency across all clients. Thus, we may focus on the $k$-th client in the following. The details of each stage are described below.
\fi

\ifshortver
\else
\textit{1. Initialization:} Before the start of the training, client $k$ and the server agree on the privacy parameters, involving the privacy budget $\epsilon_k$ and privacy relaxation $\delta_k$ for $(\epsilon_k, \delta_k)$-DP (or only $\epsilon_k$ for $\epsilon_k$-DP), and privacy-preserving noise $\mathbf{Z}\sim f$ with noise variance $\mathrm{Var}(f)>0$. 
They then align the parameters for LRSUQ as discussed in Section \ref{sec: rsuq}, including a random seed $s_k$ that serves as a source of common randomness, the lattice dimension $n$, a lattice generator matrix $\mathbf{G}$, a basic cell $\mathcal{P}$ and a function $\beta:(0,\infty)\to[0,\infty)$ satisfying
$L_{u}^{+}(f)\subseteq\beta(u)\mathcal{P}$ for $u>0$ for the randomized quantizer. Finally, both the client and the server initializes their respective PRNGs $\mathfrak{P}$ and $\tilde{\mathfrak{P}}$ with the shared random seed $s_k$, ensuring that the outputs of the two PRNGs are identical.

\textit{2. Encoding at Client:} At the end of every FL communication round, the gradient $\mathbf{X}_{t+\tau-1}^k \in \mathbb{R}^m$ is ready for transmission. Client $k$ executes the encoding function \textsc{ENCODE} to encode the gradient into finite-bit representations by utilizing LRSUQ and entropy coding. 
A detailed pseudocode description of the encoding function is given in Algorithm \ref{alg:Encode}. 
In the following we detail the quantization process:
\begin{itemize}[left=0pt, itemindent=0pt, labelsep=0.5em, labelwidth=1.2em]
    \item \textbf{Norm clipping.} Client $k$ prepares the gradient $\mathbf{X}_{t+\tau-1}^k \in \mathbb{R}^m$. First, client $k$ clips the gradient by the clipping threshold $\gamma \in (0, M]$ to $\tilde{\mathbf{X}}_{t+\tau-1}^k \in \mathcal B(0, \gamma) \subseteq \mathbb{R}^m$ so that the 2-norm of clipped gradient $\tilde{\mathbf{X}}_{t+\tau-1}^k$ is at most $\gamma$.     
    \item \textbf{Vector partitioning.} Partition the clipped gradient $\tilde{\mathbf{X}}_{t+\tau-1}^k$ into $N:= \left\lceil \frac{m}{n}\right\rceil$ sub-vectors, each $n$-dimensional.\footnote{A suitable zero-padding is required if $m$ is not a integer multiple of $n$.} Denote the collection of sub-vectors by $\{\tilde{\mathbf{X}}_{t+\tau-1,j}^k\}_{j \in \mathcal N} \subseteq \mathbb{R}^n$, where $\mathcal N:= \{1,2,\ldots, N\}$. 
    \item \textbf{LRSUQ.} Let $\mathbf{Z} \sim f$ denote the intended privacy-preserving noise random vector with mean $\mathbf{\mu} := \mathbb{E}[\mathbf{Z}] = \mathbf{0}$ and variance $\mathrm{Var}(f) := \mathrm{Var}(\mathbf{Z})$. Also, let $g(u) := \mu(L_u^+(f))$ be the pdf of the latent variable corresponds to the pdf $f$. Then for each $j \in \mathcal N$, generate a sequence of dither signal $\mathbf{V}_{t+\tau-1,j,1}^k, \mathbf{V}_{t+\tau-1,j,2}^k,\ldots \overset{iid}{\sim} \mathrm{Unif}(\mathcal P)$ and the latent random variable $U_{t+\tau-1,j}^k \sim g$ according to the PRNG $\mathfrak{P}^k$ and compute the message $\mathbf{M}_{t+\tau-1,j}^k(i) := Q_{\mathcal P}\Big(\tilde{\mathbf{X}}_{t+\tau-1,j}^k / \beta(U_{t+\tau-1,j}^k) - \mathbf{V}_{t+\tau-1,j, i}^k\Big) \in \mathbf{G} \mathbb{Z}^n$, where $\beta(\cdot)$ is the deterministic scaling function defined in Definition~\ref{def:rej_samp_quant_layer}. Until iteration $i$ that satisfies $\beta(U_{t+\tau-1,j}^k)\Big(\mathbf{M}_{t+\tau-1,j}^k(i) + \mathbf{V}_{t+\tau-1,j ,i}^k\Big) -\tilde{\mathbf{X}}_{t+\tau-1,j}^{k} \in L_{U_{t+\tau-1,j}^k}^+(f)$, take $H_{t+\tau-1,j}^k=i$ and $\mathbf{M}_{t+\tau-1,j}^k = \mathbf{M}_{t+\tau-1,j}^k(i)$. Finally, transmit $(H_{t+\tau-1,j}^k, \mathbf{M}_{t+\tau-1,j}^k)$ as a binary sequence with entropy coding. 
\end{itemize}
\fi

\ifshortver
\else
\begin{algorithm}[ht]
\caption{$\textsc{ENCODE}(\tilde{\mathbf{X}}_{t+\tau-1}^{k}, 
s_k, \mathbf{G}, \mathcal{P}, f, \beta(u), N)$} \label{alg:Encode}
\begin{algorithmic}[1]
\State \textbf{Inputs:} 
Clipped gradient vector $\tilde{\mathbf{X}}_{t+\tau-1}^{k}$,
random seed $s_k$, generator matrix $\mathbf{G}$,  basic cell $\mathcal{P}$ of $\mathbf{G}\mathbb{Z}^n$, pdf $f$, function $\beta:(0,\infty)\to[0,\infty)$ satisfying
$L_{u}^{+}(f)\subseteq\beta(u)\mathcal{P}$ for $u>0$, number of sub-vectors $N$
\State \textbf{Output:} Set of messages $\{(H_{t+\tau-1,j}^k,\mathbf{M}_{t+\tau-1,j}^k)\}_{j \in \mathcal{N}}$
\State Partition $\tilde{\mathbf{X}}_{t+\tau-1}^{k}$ to $\{\tilde{\mathbf{X}}_{t+\tau-1,j}^{k} \}_{j \in \mathcal{N}}$
\For{$j=1$ \textbf{to} $N$}
\State Sample $U_{t+\tau-1,j}^k \sim g$ where $g(u):=\mu(L_{u}^{+}(f))$ by  $\mathfrak{P}^k$ \Comment{Initiated by seed $s_k$}
\For{$i=1, 2, \ldots$} 
\Comment{Perform RSUQ}
\State \label{step:RS1} Sample 
$\mathbf{V}_{t+\tau-1,j,i}^k \sim\mathrm{Unif}(\mathcal{P})$ by $\mathfrak{P}^k$
\Comment{Initiated by seed $s_k$}
\State Find unique $\mathbf{M}_{t+\tau-1,j}^k \leftarrow Q_{\mathcal{P}}(\tilde{\mathbf{X}}_{t+\tau-1,j}^{k}/\beta(U_{t+\tau-1,j}^k)-\mathbf{V}_{t+\tau-1,j,i}^k)\in\mathbf{G}\mathbb{Z}^{n}
$ \Comment{Perform encoding}
\State \label{step:RS2} Check if $\beta(U_{t+\tau-1,j}^k)\cdot(\mathbf{M}_{t+\tau-1,j}^k+\mathbf{V}_{t+\tau-1,j,i}^k)-\tilde{\mathbf{X}}_{t+\tau-1,j}^{k}\in L_{U_{t+\tau-1,j}^k}^{+}(f)$ 
\If{Yes} \Comment{Perform rejection sampling}
\State $H_{t+\tau-1,j}^k \leftarrow i$;
 Return $(H_{t+\tau-1,j}^k,\mathbf{M}_{t+\tau-1,j}^k)$
\Else
\State Reject $i$ and repeat Step~\ref{step:RS1}-\ref{step:RS2} with $i+1$
\EndIf
\EndFor
\EndFor
\State \textbf{return} $\{(H_{t+\tau-1,j}^k,\mathbf{M}_{t+\tau-1,j}^k)\}_{j \in \mathcal{N}}$ 
\end{algorithmic}    
\end{algorithm}
\fi

\ifshortver
\else
\textit{3. Decoding at Server:} The server receives the set binary messages $\{(H_{t+\tau-1,j}^k, \mathbf{M}_{t+\tau-1,j}^k)\}_{j \in \mathcal N}$ from client $k$. First, the server generates the same realizations of dither signals $\mathbf{V}_{t+\tau-1, j, H_{t+\tau-1,j}^k}^k$ and latent random variables $U_{t+\tau-1,j}^k$ according to the random seed $s_k$ shared with client $k$. More specifically, for the $j$-th message from client $k$, the server uses the PRNG $\tilde{\mathfrak{P}}^k$ to generate the sequence $\mathbf{V}_{t+\tau-1,j,1}^k, \mathbf{V}_{t+\tau-1,j,2,}^k, \ldots, \mathbf{V}_{t+\tau-1,j,H_{t+\tau-1,j}^k}^k \overset{iid}{\sim} \mathrm{Unif}(\mathcal P)$ and $U_{t+\tau-1,j}^k \sim g$. Subsequently, the decoder computes the estimated sub-vector $\mathbf{Y}_{t+\tau-1,j}^k = \beta(U_{t+\tau-1,j}^k)(\mathbf{M}_{t+\tau-1,j}^k + \mathbf{V}_{t+\tau-1,j,H_{t+\tau-1,j}^k}^k)$. 
Subsequently, the server collects the sub-vectors $\{\mathbf{Y}_{t+\tau-1,j}^k\}_{j \in \mathcal N}$ and recovers the estimated gradient $\hat{\mathbf{X}}_{t+\tau-1}^k$.
A pseudocode description of the decoding function is given in Algorithm \ref{alg:Decode}.
Finally, the server computes the new global parameter vector according to:
\begin{equation} \label{eq:global_quan}
    \mathbf{W}_{t+\tau} \leftarrow \mathbf{W}_{t} - \eta_{t+\tau-1}\sum_{k \in \mathcal{K}} p_{k}\hat{\mathbf{X}}_{t+\tau-1}^{k},
\end{equation}
and broadcasts it to the clients.
\fi

\ifshortver
\else
In the following, we detail the privacy guarantee that is jointly ensured by the encoding-decoding steps.
\fi
The privacy enhancement is ensured by the execution of LRSUQ cooperatively between the clients and the server. Through the decoding process, the estimated sub-vectors $\{\mathbf{Y}_{t+\tau-1,j}^k\}_{j\in\mathcal N}$ are guaranteed to be noisy estimates of the clipped gradients $\{\tilde{\mathbf{X}}_{t+\tau-1,j}^k\}_{j\in\mathcal N}$. In other words, the estimated gradients recovered at the server are noisy versions of the clipped gradients transmitted by the clients, thereby implementing a privacy mechanism. Specifically, in Section \ref{sec: privacy analysis}, we prove that the global average of the estimated gradients satisfies $(\epsilon,\delta)$-DP requirement when using the Gaussian mechanism, 
\ifshortver
\else
and satisfies $\epsilon$-DP requirement when using the Laplace mechanism, respectively, 
\fi
by specifying appropriate pdfs $f$ and $g$,

\ifshortver
\else
\begin{algorithm}[t]
\caption{$\textsc{DECODE}(\{(H_{t+\tau-1,j}^k,\mathbf{M}_{t+\tau-1,j}^k)\}_{j \in \mathcal{N}}, 
s_k, \mathbf{G}, \mathcal{P}, f, \beta(u), N)$} \label{alg:Decode}
\begin{algorithmic}[1]
\State \textbf{Inputs:} Set of messages $\{(H_{t+\tau-1,j}^k,\mathbf{M}_{t+\tau-1,j}^k)\}_{j \in \mathcal{N}}$, 
random seed $s_k$, generator matrix $\mathbf{G}$, basic cell $\mathcal{P}$ of $\mathbf{G}\mathbb{Z}^n$, pdf $f$,  function $\beta:(0,\infty)\to[0,\infty)$ satisfying
$L_{u}^{+}(f)\subseteq\beta(u)\mathcal{P}$ for $u>0$, number of sub-vectors $N$
\State \textbf{Output:} Estimated gradient $\hat{\mathbf{X}}_{t+\tau-1}^{k}$
\State Use $\tilde{\mathfrak{P}}^k$ to sample $U_{t+\tau-1,j}^k \sim g$ where $g(u):=\mu(L_{u}^{+}(f))$ \Comment{Initiated by seed $s_k$}
\For{$j \in \mathcal{N}$}
\State Receive $H_{t+\tau-1,j}^k$ and $\mathbf{M}_{t+\tau-1,j}^k$ 
\State Use $\tilde{\mathfrak{P}}^k$ to sample $\mathbf{V}_{t+\tau-1,j,1},\ldots,\mathbf{V}_{t+\tau-1,j,H_{t+\tau-1,j}^k}\stackrel{iid}{\sim}\mathrm{Unif}(\mathcal{P})$ \Comment{Initiated by seed $s_k$}
\State \label{step:PP} Compute  $\mathbf{Y}_{t+\tau-1,j}^k\leftarrow\beta(U_{t+\tau-1,j}^k)(\mathbf{M}_{t+\tau-1,j}^k+\mathbf{V}_{H_{t+\tau-1,j}^k})$ \Comment{Perform decoding, inducing privacy protection}
\EndFor
\State Collect $\{\mathbf{Y}_{t+\tau-1,j}^k\}_{j \in \mathcal{N}}$ into  $\mathbf{Y}_{t+\tau-1}^k$  and set  $\hat{\mathbf{X}}_{t+\tau-1}^{k}\leftarrow 
\mathbf{Y}_{t+\tau-1}^k$ 

\State \textbf{return} $\hat{\mathbf{X}}_{t+\tau-1}^{k}$
\end{algorithmic}
\end{algorithm}
\fi

\section{Performance Analysis} \label{sec: analysis}
In this section, we theoretically study the performance of CEPAM. We characterize the privacy guarantees and compression capabilities of CEPAM in Section \ref{sec: privacy analysis} and Section \ref{sec: compression analysis}, respectively. 
Then we explore the 
\ifshortver
\else
distortion bound and convergence analysis in Section \ref{sec: distortion analysis} and Section \ref{sec: convergence analysis}, respectively. 
\fi

\ifshortver
\else
We begin by stating a useful lemma
for the subsequent analysis.

\smallskip
\begin{lemma} \label{lem:LRSUQ_error}
The LRSUQ quantization errors $\{\tilde{\mathbf{Z}}_{t+\tau-1,j}^k:= \mathbf{Y}_{t+\tau-1,j}^k - \tilde{\mathbf{X}}_{t+\tau-1,j}^{k}\}_{j \in \mathcal{N}}$ are iid (over $k$ and $j$),  follows the pdf $f$, and independent of $\tilde{\mathbf{X}}_{t+\tau-1,j}^{k}$. 
\end{lemma}

\begin{proof}
    Since we are using LRSUQ for the encoding and  decoding of the $j$-th sub-vector $\tilde{\mathbf{X}}_{t+\tau-1,j}^{k}$ where $j \in \mathcal{N}$, Proposition~\ref{prop:LRSUQ_error} implies that, regardless of the statistical models of $\tilde{\mathbf{X}}_{t+\tau-1,j}^{k}$ where $j \in \mathcal{N}$, the quantization errors $\{\tilde{\mathbf{Z}}_{t+\tau-1,j}^k\}_{j \in \mathcal{N}}$ are iid (over $k$ and $j$) and follows the pdf $f$.  
\end{proof}
\fi

\subsection{Privacy} \label{sec: privacy analysis}

By using  LRSUQ, we can construct privacy mechanisms that satisfy privacy requirements by customizing $(f,g)$. Specifically, we present a pair of $(f,g)$ that constructs Gaussian mechanism in Theorem \ref{thm:Gaussian_Priv} 
\ifshortver
\else
, and a pair of $(f,g)$ that constructs Laplace mechanism in Theorem \ref{thm:Laplace_privy}, along with their associated privacy guarantees. 
\fi
For simplicity, we assume that $p_k = 1/K$. 
The 
\ifshortver
\else 
proofs 
\fi
for Theorem \ref{thm:Gaussian_Priv} 
\ifshortver
\else 
and Theorem \ref{thm:Laplace_privy} 
\fi
can be found in \cite{ling2025communication}. 

\subsubsection{Gaussian mechanism} \label{sec:GaussMec}

If $g$ follows a chi-squared distribution with $n+2$ degrees of freedom, then $f$ follows a Gaussian distribution.
This guarantees that the global average of estimated model updates satisfies $(\epsilon, \delta)$-DP. 

\begin{theorem} \label{thm:Gaussian_Priv}
Let $\tau'=\tau-1$. Set $\tilde{U}_{t+\tau',j}^k \sim g = \chi_{n+2}^2$ and $\tilde{\mathbf{Z}}_{t+\tau',j}^k|\{\tilde{U}_{t+\tau',j}^k=u\} \sim \mathrm{Unif}(\sigma \sqrt{u} B^{n})$ for every $k, j$, 
the resulting mechanism is Gaussian.
The average $\frac{1}{K} \sum_{k \in \mathcal{K}} \hat{\mathbf{X}}_{t+\tau'}^{k}$ is a noisy estimate of the average of the clipped
gradients with    
\begin{equation} \label{eq:Gaussian_noisy_eq}
    \tfrac{1}{K} \sum_{k \in \mathcal{K}} \hat{\mathbf{X}}_{t+\tau'}^{k}= \tfrac{1}{K}\sum_{k \in \mathcal{K}}\tilde{\mathbf{X}}_{t+\tau'}^{k}+ \mathcal{N}\left(\mathbf{0},\tfrac{\sigma^2}{K}\mathbf{I}_{m}\right) , \nonumber
    \vspace{-2mm}
\end{equation}
where $\mathbf{I}_m$ is the $m \times m$ identity matrix. 
Hence, for every $\tilde{\epsilon} >0$, $\frac{1}{K} \sum_{k \in \mathcal{K}} \hat{\mathbf{X}}_{t+\tau'}^{k}$ satisfies $(\epsilon,\delta)$-DP for  $\epsilon = \log \left(1 + p(e^{\tilde{\epsilon}} - 1)\right)$ where $p = 1- \left(1-\frac{1}{|\mathcal{D}^{(k)}|}\right)^{\tau'}$ and 
\ifshortver
\else 
\begin{equation}
\begin{aligned}
    \delta = \sum_{j=1}^{\tau'} \binom{\tau'}{j}\left(\frac{1}{|\mathcal{D}^{(k)}|}\right)^j\left(1-\frac{1}{|\mathcal{D}^{(k)}|}\right)^{\tau'-j}\frac{(e^{\tilde{\epsilon}}-1)}{e^{\tilde{\epsilon}/j}-1} \times \left(\Phi\left(\frac{\tau'\gamma}{\sqrt{K}\sigma} - \frac{\sqrt{K}\tilde{\epsilon} \sigma}{2j\tau'\gamma}\right) -  e^{\tilde{\epsilon}/j} \Phi\left(-\frac{\tau'\gamma}{\sqrt{K}\sigma} - \frac{\sqrt{K}\tilde{\epsilon} \sigma}{2j\tau'\gamma}\right)\right). 
\end{aligned}
\end{equation}
\fi
\end{theorem}

\ifshortver
\else
\subsubsection{Laplace mechanism} \label{sec:LapMech}

If $g$ follows a Gamma distribution $\mathrm{Gamma}(2,1)$, then $f$ follows a Laplace distribution. This guarantees that estimated model updates satisfy $\epsilon$-DP. 

\begin{theorem} \label{thm:Laplace_privy}
Let $\tau'=\tau-1$. Set  $\tilde{U}_{t+\tau',j}^k \sim g = \mathrm{Gamma}(2,1)$  and $\tilde{Z}_{t+\tau',j}^k|\{\tilde{U}_{t+\tau',j}^k=u\} \sim \mathrm{Unif}((-bu , bu))$ for every $k$ and $j$, the resulting mechanism is Laplace.
The estimator $ \hat{\mathbf{X}}_{t+\tau'}^{k}$ is a noisy estimate of the clipped gradients $\tilde{\mathbf{X}}_{t+\tau'}^{k}$ such that    
\begin{equation} \label{eq:Laplace_noisy_eq}
    \hat{\mathbf{X}}_{t+\tau'}^{k}= \tilde{\mathbf{X}}_{t+\tau'}^{k}+ \mathrm{Lap}\left(\mathbf{0},b\mathbf{I}_{m}\right),
\end{equation}
where $\mathbf{I}_m$ is the $m \times m$ identity matrix. 
Hence, for every $ \hat{\mathbf{X}}_{t+\tau'}^{k}$ satisfies $(\epsilon,0)$-DP in one round against clients for  $\epsilon = \log \left(1 + p(e^{\tilde{\epsilon}} - 1)\right)$ where $p = 1- \left(1-\frac{1}{|\mathcal{D}^{(k)}|}\right)^{\tau'}$, and $\delta=0$ 
provided that $\tilde{\epsilon} \ge 2\tau'\gamma/b$.
\end{theorem}
\fi

\subsection{Compression} \label{sec: compression analysis}

Every client in CEPAM is required to transmit the set of message pairs $(H_{t+\tau-1,j}^k,\mathbf{M}_{t+\tau-1,j}^k)_{j\in \mathcal{N}}$ per communication round.
If we know that $\tilde{\mathbf{X}}_{t+\tau-1,j}^{k}\in \mathcal{X}$, 
then we can compress $H_{t+\tau-1,j}^k$ using the optimal prefix-free code \cite{Golomb1966enc,Gallager1975Geom} for  $\mathrm{Geom}(p(u))|\{U_{t+\tau-1,j}^k=u\}$, where $p(u):=\mu(L_{u}^{+}(f))/\mu(\beta(u)\mathcal{P})$, and compress $\mathbf{M}_{t+\tau-1,j}^k |\{U_{t+\tau-1,j}^k=u\} \in \mathcal{M} := (\mathcal{X}+L_{u}^{+}(f)-\beta(u)\mathcal{P}) \cap \mathbf{G}\mathbb{Z}^{n}
$ using $H(\lceil \log |\mathcal{M}(U)|\rceil\;|\;U)$ bits. 
Therefore, the total communication cost per communication round per client is at most $ N\cdot \left(H(\mathrm{Geom}(p(U)\;|\;U)+H(\lceil \log |\mathcal{M}(U)|\rceil\;|\;U)\right)$ bits. 

\ifshortver
\else
\subsection{Distortion Bounds} \label{sec: distortion analysis}

We consider a clipped version of FedAvg. 
This is achieved by modifying Equation~\ref{eq:global_without_quan} as follows:
\begin{equation} \label{eq:global_clipped}
    \tilde{\mathbf{W}}_{t+\tau} \leftarrow \mathbf{W}_{t} - \eta_{t+\tau-1}  \sum_{k \in \mathcal{K}} p_{k}\tilde{\mathbf{X}}_{t+\tau-1}^{k}.
\end{equation}
Representing the clipped gradient $\tilde{\mathbf{X}}_{t+\tau-1}^{k}$ of client $k$ using a finite number of bits inherently introduces distortion, as for the recovered vector $\hat{\mathbf{X}}_{t+\tau-1}^k=\tilde{\mathbf{X}}_{t+\tau-1}^{k}+\mathbf{Z}_{t+\tau-1}^k$, where $\mathbf{Z}_{t+\tau-1}^k:=\big(\tilde{\mathbf{Z}}_{t+\tau-1,1}^k,\ldots,\tilde{\mathbf{Z}}_{t+\tau-1,N}^k \big)$. 
We can then characterize the squared norm of the FL quantization error.
The proof is given in Appendix~\ref{pf:FL_error}.

\begin{proposition} \label{prop:FL_error} The FL quantization error $\mathbf{Z}_{t+\tau-1}^k$ has zero mean and  satisfies
\begin{equation}
     \mathbb{E}[\Vert\mathbf{Z}_{t+\tau-1}^k\Vert^2 \big| {\color{black}\tilde{\mathbf{x}}_{t+\tau-1}^{k}}] = {\color{black}N\mathrm{Var}(f).}
\end{equation}
\end{proposition}

At the end of each FL round, CEPAM incorporates a DP mechanism based on LRSUQ, inherently introducing distortion as mentioned above. 
This distortion emerges as the {\color{black} clipped gradient $\tilde{\mathbf{X}}_{t+\tau-1}^{k}$} is mapped into its distorted counterpart  $\hat{\mathbf{X}}_{t+\tau-1}^k$.
Consequently, the global model~\eqref{eq:global_clipped}
becomes
\begin{equation} \label{eq:global_with_quan}
    \hat{\mathbf{W}}_{t+\tau} := \mathbf{W}_{t} -  \eta_{t+\tau-1} \sum_{k \in \mathcal{K}} p_{k}\hat{\mathbf{X}}_{t+\tau-1}^{k}.
\end{equation} 
Under common assumptions used in FL analysis, we can bound the error between $\mathbf{W}_{t+\tau}$ and $\hat{\mathbf{W}}_{t+\tau}$.
We have the following bound on the error between $\mathbf{W}_{t+\tau}$ and $\hat{\mathbf{W}}_{t+\tau}$. 
The proof is given in Appendix~\ref{pf:weight_distanceB}.

\begin{proposition} \label{prop:weight_distanceB}
The mean-squared error $\mathbb{E}[\Vert \hat{\mathbf{W}}_{t+\tau}- \mathbf{W}_{t+\tau}\Vert^2]$ is upper-bounded by
\begin{equation}
     \eta_{t+\tau-1}^2N(\mathrm{Var}(f)+M^2)\sum_{k \in \mathcal{K}} p_{k}^2.
\end{equation}
\end{proposition}

\fi

\subsection{FL Convergence Analysis} \label{sec: convergence analysis}

Next, we study the FL convergence of CEPAM, considering common assumptions used in the FL  \cite{stich2018local,Li2020On, shlezinger2020uveqfed}: 

\noindent \textbf{AS1:} The $\xi_{k,j}$'s, where $j = 1, \ldots, n_k$, are i.i.d. samples in $\mathcal{D}^{(k)}$. However, different datasets $\mathcal{D}^{(k)}$ can be statistically heterogeneous and follow different distributions. 

\noindent \textbf{AS2:} Let the sample index $j_{t}^k$ be sampled from client $k$’s local data $\mathcal{D}^{(k)}$ uniformly at random. The expected squared norm of stochastic gradients is bounded by some $\theta_{k}^2>0$, 
i.e., $\mathbb{E}[\Vert \nabla F_{k}^{j_{t}^{k}}(\mathbf{W}) \Vert^2] \le \theta_{k}^2$, for all $\mathbf{W} \in \mathbb{R}^m$. 
It is clear that $\max_{k\in\mathcal K} \theta_k \leq M$.

\noindent 
\textbf{AS3:} $F_{1}, \ldots,F_{K}$ are all $L$-smooth, i.e., for all $\mathbf{x}, \mathbf{y} \in \mathbb{R}^m$, $F_{k}(\mathbf{y}) - F_{k}(\mathbf{x}) \le (\mathbf{y}-\mathbf{x})^T\nabla F_{k}(\mathbf{x})+\frac{L}{2}\Vert\mathbf{y}-\mathbf{x}\Vert^2$.  
\smallskip

\noindent
\textbf{AS4:} $F_{1}, \ldots,F_{K}$ are all $C$-strongly convex,\footnote{The extension to non-convex objective functions is left for future work.} i.e., for all $\mathbf{x}, \mathbf{y} \in \mathbb{R}^m$, $F_{k}(\mathbf{y}) - F_{k}(\mathbf{x}) \ge (\mathbf{y}-\mathbf{x})^T\nabla F_{k}(\mathbf{x})+\frac{C}{2}\Vert\mathbf{y}-\mathbf{x}\Vert^2$.

We define the following to quantify the degree of non-iid:
\ifshortver
\else 
\begin{equation}
 \psi := F(\mathbf{w}^{*}) - \sum_{k \in \mathcal{K}} \min_{\mathbf{w} \in \mathbb{R}^m} F_{k}(\mathbf{w}),   
\end{equation}
\fi
where $\mathbf{w}^{*}$ denotes the minimum of $F(\cdot)$. 
\ifshortver
\else
If the data is iid, meaning that the training data originates from the same
distribution, then $\psi$ goes to zero as the training size increases. In the case of non-i.i.d. data, $\psi$ is positive, and its magnitude reflects the heterogeneity of the data distribution.
\fi

The convergence of CEPAM for FL learning is given as follows. 
\ifshortver
\else
The proof is given in Appendix~\ref{pf:FLConvBound}.
\fi

\begin{theorem} \label{thm:FLConvBound}
Assume that AS1-4 hold and $\theta_k$ $L$, $C$ be defined therein. Set $\alpha=\tau \max\{\frac{4L}{C},1\}$ and the learning rate $\eta_t= \frac{{\color{black} \tau}}{{\color{black}C}(t+\alpha)}$ for all $t \in \mathcal T$. Then, CEPAM satisfies  
\ifshortver
\begin{align} \label{eq:FLConvBound}
   &\mathbb{E}[F(\mathbf{W}_T)]-F(\mathbf{w}^{*}) \leq \tfrac{L}{2(T+\alpha)}\max\{\tfrac{4B \tau^2(\tau+\alpha)} {C^2\alpha(\tau-1)},\alpha \Vert\mathbf{W}_0-\mathbf{w}^*\Vert\},\\
   &B:=6L\psi+ \sum_{k \in \mathcal{K}}8(\tau-1)^2p_k \theta_k^2 + p_k^2 (\theta_k^2 + N(\mathrm{Var}(f)+M^2)). \nonumber
\end{align}
\else
\begin{align} \label{eq:FLConvBound}
    \mathbb{E}[F(\mathbf{W}_T)]-F(\mathbf{w}^{*}) \leq \tfrac{L}{2(T+\alpha)}\max\left\{\tfrac{4B \tau^2(\tau+\alpha)} {C^2\alpha(\tau-1)},\alpha \Vert\mathbf{W}_0-\mathbf{w}^*\Vert\right\},
\end{align}
where
\begin{align*}
 B:=6L\psi+N(\mathrm{Var}(f)+M^2)\sum_{k \in \mathcal{K}}p_k^2+ \sum_{k \in \mathcal{K}}p_k^2 \theta_k^2+8(\tau-1)^2 \sum_{k \in \mathcal{K}}p_k \theta_{k}^{2}.
\end{align*}
\end{theorem}

A decreasing upper bound for all $t \geq 0$ can also be obtained. Between two consecutive global communication rounds, the server first broadcasts the model to all clients, and each client performs multiple local updates. Under Assumptions \textbf{AS1}-\textbf{AS4}, these local iterations constitute standard stochastic gradient steps, and standard gradient-descent analysis ensures that the expected objective value decreases across iterations, yielding convergence for all $t \geq 0$. However, to establish convergence of the model to the true minimizer, it suffices to analyze the process at the synchronous global communication rounds.

\fi

\section{Experimental Evaluations} \label{sec: exp}

In this section, we evaluate the performance of CEPAM.\footnote{ The source code used in our numerical evaluations is available at \url{https://github.com/shiushiu4863/CEPAM-Gradient}.} 
\ifshortver
\else
We begin by detailing our experimental setup, which includes the types of datasets, model architectures, and training configurations in Section \ref{sec: exp setup}. 
\fi
\ifshortver
\else
In Section \ref{sec: exp convg}, we present comprehensive experimental results to demonstrate the effectiveness of CEPAM by comparing it to several baselines that are commonly used for privacy protection and quantization in the FL framework. 
In addition, we investigate the accuracy-privacy trade-off phenomenon for  CEPAM-Gaussian \fi
\ifshortver
\else
and CEPAM-Laplace 
\fi
\ifshortver
\else
All experiments were executed on a server equipped with dual Intel Xeo nGold 6326 CPUs (48 cores and 96 threads in total), 256 GiB RAM, and two NVIDIA RTX A6000 GPUs (each with 48GB VRAM), running Ubuntu 22.04.5 LTS. The implementation was based on Python 3.13.5 and PyTorch 2.7.1+cu126, with CUDA 11.5 and NVIDIA driver version 565.57.01.

\subsection{Experimental Setup} \label{sec: exp setup}

\subsubsection{Datasets} We evaluate CEPAM on the standard image classification benchmark MNIST, which consists of $28 \times 28$ grayscale handwritten digits image divided into 60,000 training examples and 10,000 testing examples. The training examples and the testing examples are equally distributed among $K = 30$ clients. For simplicity, we set $p_k = 1/K$ for all $k \in \mathcal K$.

\subsubsection{Learning Architecture} We evaluate CEPAM using the convolutional neural network (CNN), which is composed of two convolutional layers followed by two fully-connected ones, with intermediate ReLU activations and max-pooling layers. Also, the model use a softmax output layer. There are 6422 learnable parameters for CNN.

\subsubsection{Baselines} We compare CEPAM to the following baselines:

\begin{itemize}
    \item FL: vanilla FL without any privacy or compression scheme.
    \item FL+SDQ: FL with scalar SDQ-based compression but no privacy scheme.
    \item FL+\{Gaussian, Laplace\}: FL with one-dimensional Gaussian or Laplace mechanism, but no compression scheme.
    \item FL+\{Gaussian, Laplace\}+SDQ: FL with approach that applies one-dimensional Gaussian or Laplace mechanism followed by SDQ.
    \item CEPAM-\{Gaussian, Laplace\}: CEPAM achieves privacy and compression jointly through LRSUQ to construct Gaussian or Laplace mechanisms, namely CEPAM-Gaussian or CEPAM-Laplace. In particular, we evaluate three different cases of LRSUQ that simulate Gaussian noise for dimension $n = 1,2,3$, where we use the scaled integer lattice $\alpha \mathbb{Z}^n$ with $\alpha = 10^{-3}$, and the corresponding basic cells are $\alpha (-0.5, 0.5]$, $\alpha (-0.5, 0.5]^2$, and $\alpha (-0.5, 0.5]^3$. We have chosen $\alpha = 10^{-3}$ as the precision level for quantization as it aligns with the observed precision level of the model updates during training.
\end{itemize}

\subsubsection{Training Configurations} We select the momentum SGD as the optimizer, where the momentum is set to 0.9. The local iterations $\tau$ per global epoch is set to 15. Also, we set the initial learning rate to be 0.1.

\subsubsection{Repetition Strategy} For each set of parameters and baseline method, the training process was repeated 10 times with varying random seeds. Indicators of performance, including average and 95\% confidence interval across these runs was reported in the following sections.
\fi

\subsection{FL Convergence} \label{sec: exp convg}

We report the FL convergence of CEPAM-Gaussian 
\ifshortver
\else 
and CEPAM-Laplace 
\fi
in terms of accuracy using CNN architecture and MNIST dataset.

\subsubsection{Gaussian Mechanism} 
We set $\sigma = 0.01$ and $\tilde{\epsilon} = 5.9$ for the base Gaussian mechanisms in CEPAM-Gaussian for $n = 1,2,3$ with variance $0.01^2$, $2 \times 0.01^2$ and $3\times 0.01^2$ respectively. By Theorem \ref{thm:Gaussian_Priv}, all the clients' composite Gaussian mechanisms achieve $(\epsilon = 1.45, \delta= 9.69 \times 10^{-3})$-DP. We also set $\sigma  = 0.01$ for both FL+Gaussian and FL+Gaussian+SDQ.

\ifshortver
\else
\begin{figure}[ht] 
        \centering \includegraphics[width=1\linewidth]{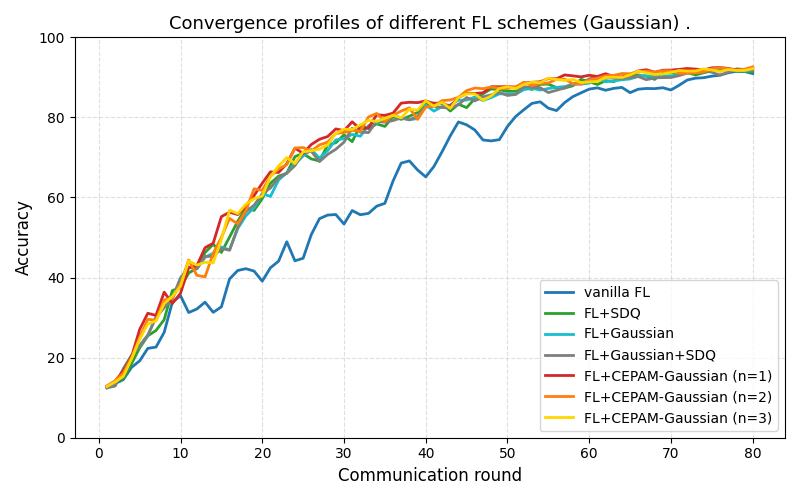}
        \caption{Convergence profile of different FL schemes for Gaussian}
        \label{fig:convergece gaussian}
        \vspace{-3pt}
\end{figure}
\fi

Figure \ref{fig:convergece gaussian} shows the validation accuracy of CEPAM-Gaussian over 80 global communication rounds using the CNN model. Note that the cases of CEPAM-Gaussian ($n=1,2,3$) outperform and achieve higher accuracy than all the other baselines. 
\ifshortver
\else
The pronounced fluctuations observed in the figure are mainly attributed to two factors. First, the server aggregates raw stochastic gradient at each communication round, instead of aggregating model differences in \cite{ling2025communication}, which inherently exhibit high variance. Second, a fixed learning rate is employed throughout the entire training phase.
\fi

\begin{table}[ht]
\centering
\renewcommand{\arraystretch}{0.8}
\begin{tabular}{ccc}
\toprule
Baselines & Accuracy (\%) & Average SNR (dB) \\
\midrule
FL                & $93.24 \pm 0.31$ & $\infty$ \\
FL+SDQ            & $93.28 \pm 0.63$ & $84.22$ \\
FL+Gaussian       & $93.46 \pm 0.70$ & $24.52$ \\
FL+Gaussian+SDQ   & $93.18 \pm 0.59$ & $24.65$ \\
CEPAM-Gaussian ($n=1$) & $94.11 \pm 0.49$ & $8.15$ \\
CEPAM-Gaussian ($n=2$) & $94.06 \pm 0.45$ & $15.26$ \\
CEPAM-Gaussian ($n=3$) & $93.86 \pm 0.36$ & $19.76$ \\
\bottomrule \\
\end{tabular}
\vspace{-3pt}
\caption{Test Accuracy for Gaussian on MNIST}
\label{table:AccGaussian}
\vspace{-3pt}
\end{table}

In Table \ref{table:AccGaussian}, we report the test accuracy of CEPAM-Gaussian with their 95\% confidence intervals and the average SNR ratio (in dB). CEPAM-Gaussian in different dimensions all demonstrate better accuracy, achieving an improvement of 0.4-0.9\% in accuracy compared to the other baselines. Furthermore, CEPAM-Gaussian of different dimension have approximately the same accuracy. This is because the added noise per dimension is consistent, i.e., $\mathrm{Var}(f) / n = \sigma^2$. However, higher-dimension LRSUQ has a better compression ratio than scalar counterparts 
\ifshortver
\else
\cite{hasircioglu2023communication, yan2023layered}, 
\fi
due to its nature as a vector quantizer.

\ifshortver
\else
\subsubsection{Laplace Mechanism}

We set $b = 0.01$ and $\tilde{\epsilon} = 3000$ for the base Laplace mechanism in CEPAM-Laplace with a variance $2 \times 0.01^2$, and the composite Laplace mechanism achieves $(\epsilon = 2995)$-DP. We also set $b = 0.01$ for both FL+Laplace and FL+Laplace+SDQ. Note that a higher privacy budget is required to ensure $\epsilon$-DP for CEPAM-Laplace. However, we might achieve a lower privacy budget by relaxing $\delta$. The trade-off of this relaxation is left for future study.

\begin{figure}[ht]
        \centering
        \includegraphics[width=1\linewidth]{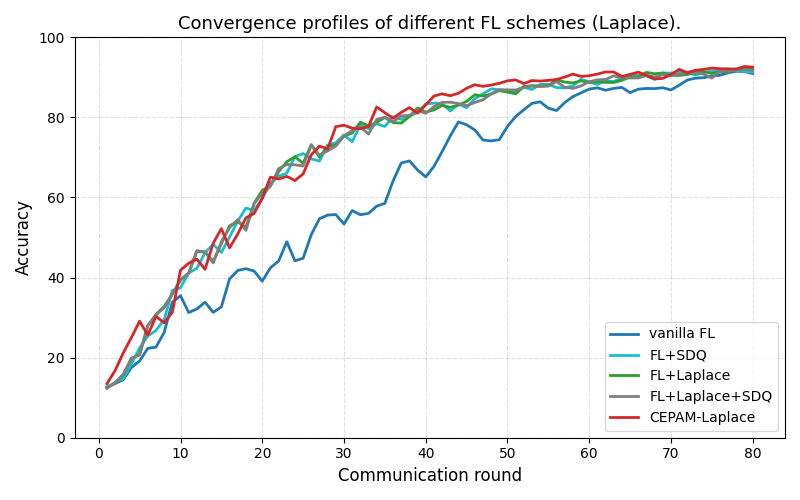}
        \caption{Convergence profile of different FL schemes for Laplace}
        \label{fig:convergece laplace}
        \vspace{-3pt}
\end{figure} 

Figure \ref{fig:convergece laplace} shows the validation accuracy of CEPAM-Laplace over global epochs using the CNN model. Similar to CEPAM-Gaussian, we observe that CEPAM-Laplace outperforms other baselines.

\begin{table}[ht]
    \centering
    \begin{tabular}{ccc}
    \toprule
    Baselines & Accuracy (\%) & Average SNR (dB) \\
    \midrule
    FL                & $93.24 \pm 0.31$ & $\infty$ \\
    FL+SDQ            & $93.28 \pm 0.63$ & $84.22$ \\
    FL+Laplace        & $93.43 \pm 0.54$ & $20.96$ \\
    FL+Laplace+SDQ    & $93.25 \pm 0.59$ & $21.21$ \\
    CEPAM-Laplace     & $94.32 \pm 0.40$ & $-1.582$ \\
    \bottomrule \\
    \end{tabular}
    \vspace{-3pt}
    \caption{Test Accuracy for Laplace on MNIST}
    \label{table:AccLaplace}
    \vspace{-3pt}
\end{table}

In Table \ref{table:AccLaplace}, we also report the performance of test accuracy of CEPAM-Laplace with their 95\% confidence intervals and the average SNR ratio (in dB). We observe that CEPAM-Laplace achieves an improvement of about 0.8-1.1\% in accuracy compared to the other baselines.
\fi

It is worth noting that adding a minor level of distortion during training deep models  can potentially enhance the performance of the converged model \cite{An1996AddingNoise}. 
This finding corroborates in our experimental results, as 
\ifshortver
\else
both 
\fi
CEPAM-Gaussian 
\ifshortver
\else
and CEPAM-Laplace 
\fi
achieves slightly better accuracy than vanilla FL.

\subsection{Privacy-Accuracy Trade-off} \label{sec: exp tradeoff}
\subsubsection{Gaussian Mechanism}

We evaluate the privacy-accuracy trade-off with CEPAM-Gaussian. In the following, the parameters are computed according to Theorem \ref{thm:Gaussian_Priv}. We set $\delta = 0.015$, we conducted the experiments by varying privacy budget $\epsilon$ from $0.5$ to $5$, with step size $0.5$. The corresponding noise levels $\sigma$ 
\ifshortver
\else 
are $\{0.13881, 0.08567, 0.06058, 0.04447, 0.03274, 0.02362, 0.01623, $ $ 0.01011, 0.00494, 0.00052\}$ respectively. For each set of parameters, the experiments are simulated 10 times with different random seeds, and the results are averaged. 
\fi

\ifshortver
\else
\begin{figure}
        \centering
        \includegraphics[width=1\linewidth]{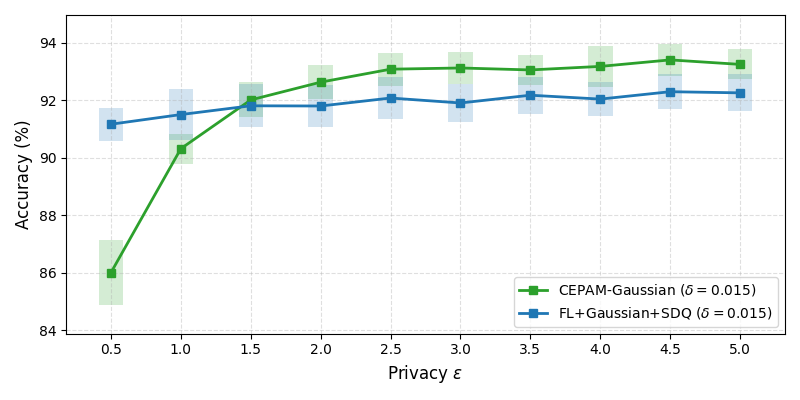}
        \caption{Accuracy and Privacy Trade-off of Gaussian}
        \label{fig: tradeoff_gaussian}
        \vspace{-3pt}
\end{figure}
\fi

Figure \ref{fig: tradeoff_gaussian} illustrates the trade-off between learning performance, measured in terms of test accuracy, and the privacy budget between CEPAM-Gaussian and the simple Gaussian-mechanism-then-quantize approach (Gaussian+SDQ).

Overall, CEPAM-Gaussian has a better performance than Gaussian+SDQ. The figure demonstrates that as more privacy budget is allocated, higher test accuracy can be achieved with CEPAM-Gaussian. However, there is a point of diminishing returns; one the privacy budget reaches a threshold ($\epsilon \approx 2.5$), the increase in test accuracy by further increasing the privacy budget becomes limited.

\ifshortver
\else
\subsubsection{Laplace Mechanism}

Similarly, we evaluate the privacy-accuracy trade-off with CEPAM-Laplace. In the following, the parameters are computed according to Theorem \ref{thm:Laplace_privy}. We conducted the experiments by varying privacy budget $\epsilon$ from $500$ to $5000$, with step size $500$. The corresponding noise levels $b$ are \{0.05944, 0.029859, 0.019937, 0.014965, 0.011977, 0.009984, 0.00856, 0.007491, 0.00666, 0.005994\} respectively. Again, for each set of parameters, the experiments are simulated 10 times with different random seeds, and the results are averaged.

\begin{figure}[ht]
        \centering
        \includegraphics[width=1\linewidth]{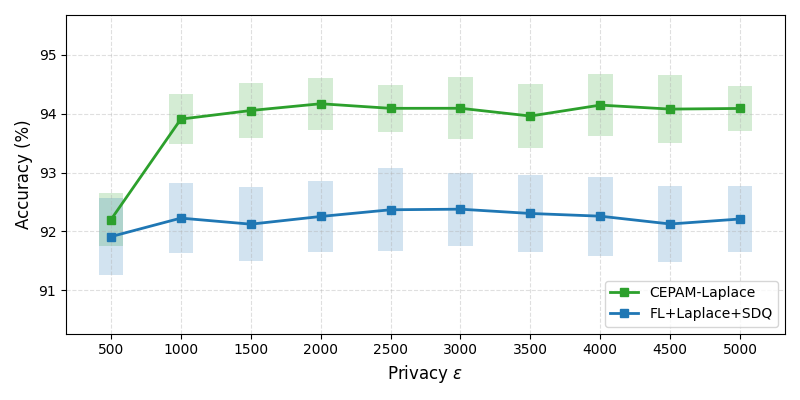}
        \caption{Accuracy and Privacy Trade-off of Laplace}
        \label{fig: tradeoff_laplace}
        \vspace{-3pt}
\end{figure}

Figure \ref{fig: tradeoff_laplace} illustrates the trade-off between learning performance, measured in terms of test accuracy, and the privacy budget between CEPAM-Laplace and the simple Laplace-mechanism-then-quantize approach (Laplace+SDQ).

Overall, CEPAM-Laplace outperforms Gaussian+SDQ. A similar diminishing return phenomenon can be observed in this case, as the privacy budget reaches a threshold ($\epsilon \approx 1500$), the increase in test accuracy by further increasing the privacy budget becomes limited.

It is worth noting that the test accuracy slightly decreased when we further increase the privacy budget ($\epsilon \approx 4000$), i.e., lowering the noise level $b$. This finding again coincides with the finding that adding a minor level of distortion can potentially enhance  performance.

\fi

\ifshortver
\else
\section{Conclusion and Discussion} \label{sec: conclusion}
In this paper, we further investigated CEPAM, which is designed to achieve communication efficiency and privacy protection simultaneously in the FL framework.
We proved the FL convergence bound for CEPAM. We also conducted various experiments on a modified version of CEPAM, including convergence profiles and accuracy-privacy trade-off behavior. The experimental results confirmed our theoretical findings and demonstrated that both CEPAM-Gaussian and CEPAM-Laplace lead to an enhancement in test accuracy compared to several other commonly used baselines in the FL framework.

For potential future work, several interesting directions emerge. One could be to conduct a convergence analysis for the nonconvex case and extend the privacy analysis to other subsampling mechanisms. Another promising avenue would be to adapt CEPAM for personalized federated learning and rigorously evaluate its performance and efficacy.
\fi

\ifshortver
\else
\section*{Acknowledgment}
The authors thank Ms. Youqi Wu for developing the original framework for the simulations in Section \ref{sec: exp}.
\fi

\ifshortver
\else

\appendices

\section{Proof of Proposition~\ref{prop:FL_error} \label{pf:FL_error}}

Fix $\tilde{\mathbf{X}}_{t+\tau-1,j}^{k}=\tilde{\mathbf{x}}_{t+\tau-1,j}^{k}$.
By Lemma~\ref{lem:LRSUQ_error}, the LRSUQ quantization errors $\{\tilde{\mathbf{Z}}_{t+\tau-1,j}^k:= \mathbf{Y}_{t+\tau-1,j}^k - \tilde{\mathbf{x}}_{t+\tau-1,j}^{k}\}_{j \in \mathcal{N}}$ are iid (over $k$ and $j$) and follows the pdf $f$. 
Hence, 
\begin{align*}
\mathbb{E}[\Vert\mathbf{Z}_{t+\tau-1}^k\Vert^2 
\big| {\color{black}\tilde{\mathbf{x}}_{t+\tau-1}^{k}}] 
&=  \sum_{j \in \mathcal{N}}\mathbb{E}[\Vert\tilde{\mathbf{Z}}_{t+\tau-1,j}^k\Vert^2]\\
&\stackrel{(a)}{=} \sum_{j \in \mathcal{N}}\mathrm{Var}(f)\\
&=  N\mathrm{Var}(f), 
\end{align*}
where (a) is because the mean of $\tilde{\mathbf{Z}}_{t+\tau-1,j}^k \sim f$ is $\mathbf{0}$. This completes the proof.

\section{Proof of Proposition~\ref{prop:weight_distanceB}\label{pf:weight_distanceB}}
Let $\{\mathbf{W}_{t+\tau,j}\}_{j \in \mathcal{N}}$ and $\{\hat{\mathbf{W}}_{t+\tau,j}\}_{j \in \mathcal{N}}$ denote the partitions of $\mathbf{W}_{t+\tau}$ and $\hat{\mathbf{W}}_{t+\tau}$ into $N$ distinct $n$-dimensional sub-vectors, respectively. 
We have, for $j \in \mathcal{N}$,
\begin{equation*}
\hat{\mathbf{W}}_{t+\tau,j} = \mathbf{W}_{t,j} - \eta_{t+\tau-1}\sum_{k \in \mathcal{K}} p_{k} \hat{\mathbf{X}}_{t+\tau-1,j}^{k},
\end{equation*}
and 
\begin{equation*}
\mathbf{W}_{t+\tau,j} = \mathbf{W}_{t,j} - \eta_{t+\tau-1}\sum_{k \in \mathcal{K}} p_{k} \mathbf{X}_{t+\tau-1,j}^{k}.
\end{equation*}
Hence, 
\begin{align*}
\mathbb{E}[\Vert \hat{\mathbf{W}}_{t+\tau}- \mathbf{W}_{t+\tau}\Vert^2] 
&=\mathbb{E}\Big[\Vert \eta_{t+\tau-1}\sum_{j \in \mathcal{N}} \sum_{k \in \mathcal{K}}p_{k}  (\hat{\mathbf{X}}_{t+\tau-1,j}^k - \mathbf{X}_{t+\tau-1,j}^k) \Vert^2\Big]  \\
&\stackrel{(a)}{\le} \eta_{t+\tau-1}^2\sum_{k \in \mathcal{K}} p_{k}^2 \sum_{j \in \mathcal{N}} \mathbb{E} \Big[\Vert \hat{\mathbf{X}}_{t+\tau-1,j}^k - \mathbf{X}_{t+\tau-1,j}^k \Vert^2\Big]  \\
&  = \eta_{t+\tau-1}^2\sum_{k \in \mathcal{K}} p_{k}^2 \sum_{j \in \mathcal{N}} \mathbb{E} \Big[\Vert(\hat{\mathbf{X}}_{t+\tau-1,j}^k - \tilde{\mathbf{X}}_{t+\tau-1,j}^k) + (\tilde{\mathbf{X}}_{t+\tau-1,j}^k - \mathbf{X}_{t+\tau-1,j}^k) \Vert^2\Big]  \\
& \leq \eta_{t+\tau-1}^2\sum_{k \in \mathcal{K}} p_{k}^2 \sum_{j \in \mathcal{N}} \left(\mathbb{E} \Big[\Vert(\hat{\mathbf{X}}_{t+\tau-1,j}^k - \tilde{\mathbf{X}}_{t+\tau-1,j}^k)\Vert^2\Big]+ \mathbb{E}\Big[\Vert(\tilde{\mathbf{X}}_{t+\tau-1,j}^k - \mathbf{X}_{t+\tau-1,j}^k) \Vert^2\Big]\right) \\
& \leq \eta_{t+\tau-1}^2\sum_{k \in \mathcal{K}} p_{k}^2 \sum_{j \in \mathcal{N}} \left( \mathbb{E} \Big[\Vert \tilde{\mathbf{Z}}_{t+\tau-1,j}^k \Vert^2\Big] + M^2 \right) \\
&\stackrel{(b)}{=} \eta_{t+\tau-1}^2N \sum_{k \in \mathcal{K}} p_{k}^2 (\mathrm{Var}(f) + M^2), 
\end{align*}
where (a) is by the triangle inequality, and (b) is by the argument of Proposition~\ref{prop:FL_error}.
This completes the proof.

\section{Proof of Theorem~\ref{thm:FLConvBound}\label{pf:FLConvBound}}

Our proof adopts similar strategies to those in \cite{stich2018local,Li2020On,shlezinger2020uveqfed}, supplemented by additional arguments to address the LRSUQ quantization error. The distinctive features of the quantization error introduced by LRSUQ, as discussed in Lemma~\ref{lem:LRSUQ_error}, enable us to establish the convergence property of CEPAM. 
We include the proof for the sake of completeness.

Recall that $\mathbf{Z}_{t-1}^k$ where $t \in \mathcal{T}_T$ is the FL quantization error, which is independent of the clipped gradient $\tilde{\mathbf{X}}_{t-1}^k$ by Lemma~\ref{lem:LRSUQ_error}.
Let $\mathbf{E}_{t-1}^k := \mathbf{Z}_{t-1}^k + \mathbf{C}_{t-1}^k$ if $t\in \mathcal{T}_T$ and $\mathbf{E}_{t}^k:=\mathbf{0}$ otherwise, where $\mathbf{C}_{t-1}^k$ is the clipping error, then the update of CEPAM can also be described alternatively (compared to~\eqref{eq:local_sgd} and \eqref{eq:global_with_quan}) as follows:
\begin{equation} \label{eq:local_sgd_alt}
\mathbf{W}_{t+1}^k =
\begin{cases}
\mathbf{W}_{t}^{k} - \eta_{t} \nabla F_{k}^{j_{t}^{k}}(\mathbf{W}_t^{k}) + \mathbf{E}_{t}^k      & \text{if } t+1\notin \mathcal{T}_T, \\
\sum_{k' \in \mathcal{K}}p_{k'}\left(\mathbf{W}_{t+1-\tau}^{k'} - \eta_{t} \nabla F_{k'}^{j_{t}^{k'}}(\mathbf{W}_t^{k'}) + \eta_{t}\mathbf{E}_{t}^{k'}\right) & \text{if } t+1\in \mathcal{T}_T.
\end{cases}
\end{equation}
By applying the strategy outlined in \cite{stich2018local} and adapting it to Assumption 1, i.e.,  heterogeneous dataset as discussed in \cite{Li2020On}, we define a virtual sequence $(\mathbf{W}_{t}')_{t \in \{0\}\cup[T]}$, where $[T]:=\{1,2,\ldots,T\}$, from the FL model weights $(\mathbf{W}_t^k)_{t \in \{0\}\cup[T]}$ as follows: 
\begin{equation}
   \mathbf{W}_{t}' = \sum_{k \in \mathcal{K}} p_k \mathbf{W}_t^k,  
\end{equation}
which coincides with $\mathbf{W}_t^k$ when $t \in \mathcal{T}_T$.
This sequence can be demonstrated to behave similarly to mini-batch SGD with a batch size of $\tau$ and being bounded within a bounded distance of  $(\mathbf{W}_t^k)_{t \in \{0\}\cup[T]}$,  by appropriately configuring the step size $\eta_t$ \cite{stich2018local}.
For convenience, define the \emph{averaged full gradients} and the \emph{averaged perturbed stochastic gradients} as 
\begin{align}
    \bar{\mathbf{J}}_{t} &:= \sum_{k \in \mathcal{K}}p_k \nabla F_{k}(\mathbf{W}_t^k), \\
    \mathbf{J}_{t} &:=  \sum_{k \in \mathcal{K}}p_k \left(\nabla F_{k}^{j_{t}^{k}}(\mathbf{W}_t^k)-\mathbf{E}_{t}^k\right),
\end{align}
respectively.
Note that since the FL quantization error $\mathbf{Z}_{t+\tau{\color{blue}-1}}$ has zero mean by Proposition~\ref{prop:FL_error} and the sample indices $j_{t}^k$ are independent and uniformly distributed, it follows that $\mathbb{E}[\mathbf{J}_{t}]=\bar{\mathbf{J}}_{t}$. 
Moreover, the virtual sequence satisfies $\mathbf{W}_{t+1}'=\mathbf{W}_{t}'-\eta_t\mathbf{J}_{t}$ if $t+1\notin \mathcal{T}_T$ and $\mathbf{W}_{t+1}'=\mathbf{W}_{t+1-\tau}'-\eta_t\mathbf{J}_{t}$ if $t+1 \in \mathcal{T}_T$ (see Figure~\ref{fig: flow diagram}).

Therefore, the resulting model is equivalent to the model discussed in \cite[Appendix A]{Li2020On} or \cite[Appendix C]{shlezinger2020uveqfed}. 
This implies that we can apply the following result (one step SGD for heterogeneity dataset) from \cite[Lemma 1]{Li2020On} for $t+1 \notin \mathcal{T}_T$:
\begin{equation} \label{eq:onestepSGD1}
 \mathbb{E}[\Vert \mathbf{W}_{t+1}'-\mathbf{w}^*\Vert^2] \le (1 - \eta_t C)\mathbb{E}[\Vert \mathbf{W}_{t}'-\mathbf{w}^*\Vert^2] + 6L\eta_{t}^2\psi+\eta_t^2\mathbb{E}[\Vert \mathbf{J}_{t}-\bar{\mathbf{J}}_{t}\Vert^2]
 +2\mathbb{E}\left[\sum_{k \in \mathcal{K}}p_k\Vert \mathbf{W}_{t}'-\mathbf{W}_{t}^k\Vert^2\right],
\end{equation}
provided that $\eta_t \le \frac{1}{4L}$ and Assumptions 3-4 hold, and for $t+1 \in \mathcal{T}_T$:
\begin{equation} \label{eq:onestepSGD2}
 \mathbb{E}[\Vert \mathbf{W}_{t+1}'-\mathbf{w}^*\Vert^2] \le (1 - \eta_t C)\mathbb{E}[\Vert \mathbf{W}_{t+1-\tau}'-\mathbf{w}^*\Vert^2] + 6L\eta_{t}^2\psi+\eta_t^2\mathbb{E}[\Vert \mathbf{J}_{t}-\bar{\mathbf{J}}_{t}\Vert^2]
 +2\mathbb{E}\left[\sum_{k \in \mathcal{K}}p_k\Vert \mathbf{W}_{t}'-\mathbf{W}_{t}^k\Vert^2\right],
\end{equation}
provided that $\eta_t \le \frac{1}{4L}$ and Assumptions 3-4 hold.

The third term in RHS of~\eqref{eq:onestepSGD1} and of~\eqref{eq:onestepSGD2} can be upper-bounded by using the following lemma, where the proof is given in Appendix~\ref{pf:virtual error}:

\begin{lemma} \label{lem:virtual error}
    When \textbf{AS2} holds, we have
    \begin{align}
        \mathbb{E}[\|\bar{\mathbf{J}}_t  - \mathbf{J}_t\|^2] \leq \sum_{k \in \mathcal K} p_k^2 (\theta_k^2 + N(\Var(f)+M^2)). \label{eq:grad_bound}
    \end{align}
\end{lemma}

Furthermore, the last term in RHS of~\eqref{eq:onestepSGD1} and of~\eqref{eq:onestepSGD2} can be upper-bounded by using the following lemma, where the proof is given in Appendix~\ref{pf:bound_deviate}:

\begin{lemma} \label{lem:bound_deviate}
Assume that $\eta_t$ is non-increasing and $\eta_t \le 2\eta_{t+\tau}$ for all $t\ge 0$. When Assumption 2 holds, we have 
\begin{equation} \label{ieq:bound_deviate}
    \mathbb{E}\left[\sum_{k \in \mathcal{K}}p_k\Vert \mathbf{W}_{t}'-\mathbf{W}_{t}^k\Vert^2\right] \le 4(\tau-1)^2 \eta_{t}^2\sum_{k \in \mathcal{K}}p_k \theta_{k}^{2}.
\end{equation}
\end{lemma}

Therefore, by defining $\delta_t:=\mathbb{E}[\Vert\mathbf{W}_{t}'-\mathbf{w}^{*}\Vert]$ and substituting \eqref{eq:grad_bound} and \eqref{ieq:bound_deviate} into~\eqref{eq:onestepSGD1} and~\eqref{eq:onestepSGD2}, we have the following recursive bounds: 
\begin{equation} \label{eq:recursive_bound}
    \delta_{t+1} \le (1-C\eta_t )\delta_{t}+B\eta_{t}^2, \quad \text{if } t+1\notin \mathcal{T}_T, 
\end{equation}
and
\begin{equation} \label{eq:recursive_bound2}
    \delta_{t+1} \le (1-C\eta_t )\delta_{t+1-\tau}+B\eta_{t}^2, \quad \text{if } t+1\in \mathcal{T}_T,
\end{equation}
where 
\begin{equation}
B:= 6L\psi+
     \sum_{k \in \mathcal{K}}p_k^2\left(N(\mathrm{Var}(f) +M^2)+\theta_{k}^{2}\right)+8(\tau-1)^2 \sum_{k \in \mathcal{K}}p_k \theta_{k}^{2}.    
\end{equation}

By setting the learning rate $\eta_t$ and the FL system parameters, combined with some smoothness assumption of the local objective functions, we can prove (\ref{eq:FLConvBound}).

Specifically, we set the diminishing learning rate $\eta_t := \frac{\zeta}{t+\alpha}$ for some $\zeta > 0$ and $\alpha \geq \max\{4L\zeta, \tau\}$ such that $\eta_t \leq \frac{1}{4L}$ and $\eta_t \leq 2 \eta_{t+\tau}$ to ensure the validity of Lemmas \ref{lem:virtual error} and \ref{lem:bound_deviate}. Under this setup, we will demonstrate the existence of a finite $\nu, \alpha > 0$ such that 
\begin{align*}
    \delta_{t} \leq \frac{\nu}{t+\alpha}
\end{align*}
for all $t \in \mathcal T$. We prove it by induction. For $t = 0$, the above holds if $\nu \geq \alpha \delta_0$. Assuming that the above holds for some $t = m \tau$, where $m \in \mathbb{Z}_{+}$, it follows by inductively applying (\ref{eq:recursive_bound}) that 

\begin{align*}
    \delta_{t+\tau} & \leq \delta_t \prod_{i=t}^{t+\tau-1} (1 - C\eta_i) + B \sum_{i=1}^{t+\tau-1} \eta_i^2 \cdot \underbrace{\prod_{i=t}^{t+\tau-1} (1 - C\eta_i)}_{\leq 1} \\
    & \leq \delta_t \prod_{i=t}^{t+\tau-1} (1 - C\eta_i) + B \sum_{i=1}^{t+\tau-1} \eta_i^2 \\
    & \leq \delta_t \prod_{i=1}^{t+\tau-1} (1 - C \eta_i) + 4B \tau \eta_t^2 \\
    & \leq \delta_t\exp\left( - C \sum_{i=1}^{t+\tau-1} \eta_i\right) + 4B \tau \eta_t^2 \\
    & = \delta_t\exp\left( - C \sum_{i=1}^{t+\tau-1} \frac{\zeta}{i+\alpha}\right) + 4B \tau \eta_t^2 \\
    & \leq \delta_t\exp\left( - \frac{C\zeta\tau}{t+\tau+\alpha}\right) + 4B \tau \eta_t^2 \\
    & \leq \delta_t \left(1 - \frac{C\zeta\tau}{2(t+\tau+\alpha)}\right) + \frac{4B\tau\zeta^2}{(t+\alpha)^2},
\end{align*}
where we used the inequality $e^{-x} \leq 1-x/2$ for $x \in [0,1]$. Using our induction hypothesis that $\delta_{m\tau} \leq \frac{\nu}{m\tau + \alpha}$,
\begin{align*}
    \delta_{(m+1)\tau} & = \delta_{m\tau + \tau} \\
    & \leq \frac{\nu}{m\tau + \alpha}\delta_t \left(1 - \frac{C\zeta\tau}{2(t+\tau+\alpha)}\right) + \frac{4B\tau\zeta^2}{(t+\alpha)^2} \\
    & = \frac{\nu}{m\tau+\alpha} - \frac{\nu}{(m+1)\tau+\alpha} - \frac{\nu C \zeta \tau}{2(m\tau+\alpha)((m+1)\tau + \alpha)} + \frac{4B\tau\zeta^2}{(m\tau + \alpha)^2} + \frac{\nu}{(m+1)\tau+\alpha} \\
    & = \frac{\nu\tau}{(m\tau + \alpha)((m+1)\tau+ \alpha)}\left(1 - \frac{C\zeta}{2}\right) + \frac{4B\tau \zeta^2}{(m\tau+\alpha)^2} + \frac{\nu}{(m+1)\tau + \alpha}.
\end{align*}
Therefore, $\delta_{(m+1)\tau} \leq \frac{\nu}{(m+1)\tau + \alpha}$ holds when
\begin{align}
    \frac{4B\tau \zeta^2}{(m\tau+\alpha)^2} & \leq \frac{\nu\tau}{(m\tau + \alpha)((m+1)\tau+ \alpha)}\left( \frac{C\zeta}{2}- 1\right) \notag \\
    \frac{4B\zeta^2}{m\tau+\alpha} & \leq \frac{\nu}{(m+1)\tau + \alpha} \left(\frac{C\zeta}{2} - 1 \right) \notag \\
    \nu & \geq \frac{(m+1)\tau + \alpha}{m\tau+\alpha} \cdot \frac{4B\zeta^2}{\frac{C\zeta}{2}-1} \label{eq:rec_bound requirement}
\end{align}
Note that for $m \geq 0$,
\begin{align*}
    \frac{m\tau + \tau + \alpha}{m\tau + \alpha} \geq \frac{\tau+\alpha}{\alpha},
\end{align*}
it suffices to choose $\nu \geq \frac{4B\zeta^2(\tau+\alpha)}{\alpha\left(\frac{C\zeta}{2}-1\right)}$. 

Finally, by the smoothness property of the local objective functions
\begin{align}
    \mathbb{E}[F(\mathbf{W}_T)] - F(\mathbf{w}^*) \leq \frac{L}{2}\delta_T \leq \frac{L\nu}{2(T+\alpha)},
\end{align}
provided that $\nu \geq \max\left\{\alpha \delta_0, \frac{4B\zeta^2(\tau+\alpha)}{\alpha\left(\frac{C\zeta}{2}-1\right)} \right\}$, $\alpha \geq \{4L\zeta, \tau\}$ and $\zeta > 0$. In particular, setting $\zeta = \frac{\tau}{C}$ implies that $\alpha \geq \tau \max\left\{\frac{4L}{C} , 1 \right\}$ and $\nu \geq \max\left\{\frac{4B \tau^2(\tau+\alpha)} {C^2\alpha(\tau-1)},\alpha \Vert\mathbf{W}_0-\mathbf{w}^*\Vert\right\}$. 

We emphasize that the above analysis is for synchronous communication rounds $t \in \mathcal T$. A decreasing upper bound for general $t \geq 0$ can be obtained in a similar manner by utilizing standard gradient-descent analysis.

\section{Proof of Lemma~\ref{lem:virtual error} \label{pf:virtual error}}

Note that 
    \begin{align*}
        \bar{\mathbf{J}}_t  - \mathbf{J}_t & = \sum_{k \in \mathcal K} p_k \nabla F_k(\mathbf{W}_t^k) - p_k \left(\nabla F_k^{j_t^k} (\mathbf{W}_t^k) - \mathbf{E}_t^k\right) \\
        & =  \sum_{k \in \mathcal K} p_k (\nabla F_k(\mathbf{W}_t^k) - \nabla F_k^{j_t^k}(\mathbf{W}_t^k))  + \sum_{k \in \mathcal K}p_k \mathbf{E}_t^k.
    \end{align*}
    Taking squared norm and expectation on both sides yield 
    \begin{align*}
        \mathbb{E}[\|\bar{\mathbf{J}}_t  - \mathbf{J}_t\|^2] & = \mathbb{E} \left[\left\|\sum_{k \in \mathcal K} p_k (\nabla F_k(\mathbf{W}_t^k) - \nabla F_k^{j_t^k}(\mathbf{W}_t^k))  + \sum_{k \in \mathcal K}p_k \mathbf{E}_t^k\right\|^2\right] \\
        & = \mathbb{E} \left[\left\|\sum_{k \in \mathcal K} p_k (\nabla F_k(\mathbf{W}_t^k) - \nabla F_k^{j_t^k}(\mathbf{W}_t^k))\right\|^2\right] + \mathbb{E}\left[\left\| \sum_{k\in\mathcal K} p_k \mathbf{E}_t^k\right\|^2\right] \\
        & \quad + 2 \underbrace{\mathbb{E} \left\langle \sum_{k \in \mathcal K} p_k (\nabla F_k(\mathbf{W}_t^k) - \nabla F_k^{j_t^k} (\mathbf{W}_t^k)), \sum_{k'\in\mathcal K} p_{k'} \mathbf{E}_t^{k'} \right\rangle }_{=0 \text{ as } \mathbb{E}(\mathbf{E}_t^k) = 0}  \\
        & \leq \mathbb{E} \left[\sum_{k \in \mathcal K}\left\| p_k (\nabla F_k(\mathbf{W}_t^k) - \nabla F_k^{j_t^k}(\mathbf{W}_t^k))\right\|^2\right] + \mathbb{E}\left[\sum_{k\in\mathcal K} \left\| p_k \mathbf{E}_t^k\right\|^2\right] \\
        & =  \sum_{k \in \mathcal K} p_k^2\mathbb{E}\left[\left\| (\nabla F_k(\mathbf{W}_t^k) - \nabla F_k^{j_t^k}(\mathbf{W}_t^k))\right\|^2\right] + \sum_{k\in\mathcal K}p_k^2 \mathbb{E}\left[\left\| \mathbf{E}_t^k\right\|^2\right] \\
        & \overset{\mathrm{(a)}}{\leq} \sum_{k \in \mathcal K} p_k^2\mathbb{E}\left[\left\| \nabla F_k^{j_t^k}(\mathbf{W}_t^k)\right\|^2\right] + \sum_{k\in\mathcal K}p_k^2 \mathbb{E}\left[\left\| \mathbf{E}_t^k\right\|^2\right] \\
        & \overset{\mathrm{(b)}}{=} \sum_{k \in \mathcal K} p_k^2\mathbb{E}\left[\left\| \nabla F_k^{j_t^k}(\mathbf{W}_t^k)\right\|^2\right] + \sum_{k\in\mathcal K}p_k^2 \mathbb{E}\left[\left\| \mathbf{Z}_{t-1}^k + \mathbf{C}_{t-1}^k\right\|^2\right] \\
        & \overset{\mathrm{(c)}}{\leq} \sum_{k \in \mathcal K} p_k^2 \theta_k^2 + \sum_{k\in\mathcal K} p_k^2 N(\Var(f) + M^2) \\
        & = \sum_{k\in\mathcal K} p_k^2(\theta_k^2 + N(\Var(f) + M^2)),
    \end{align*}
    where (a) holds since $\mathbb{E}[\nabla F_k^{j_t^k}(\mathbf{W})] = \nabla F_k(\mathbf{W})$, implying $\mathbb{E}\left[\|\nabla F_k(\mathbf{W}_t^k) - \nabla F_k^{j_t^k}(\mathbf{W}_t^k)\|^2\right] \leq \mathbb{E}\left[\|\nabla F_k^{j_t^k}(\mathbf{W}_t^k)\|^2\right]$, (b) is due to norm clipping, (c) holds by \textbf{AS2} and $\mathbb{E}[ \|\mathbf{Z}_t^k\|^2] \in \{0, N \Var(f)\}$, hence $\mathbb{E}[ \|\mathbf{E}_t^k\|^2] \leq N(\Var(f) + M^2)$ for all $t$.

\section{Proof of Lemma~\ref{lem:bound_deviate} \label{pf:bound_deviate}}

Since FedAvg contains $\tau$ steps in one communication round, it follows that for $t \ge 0$, there exists a $t_0 \le t$ such that $t_0 \in \mathcal{T}_T$, $t - t_0 \le \tau-1$ and $\mathbf{W}_{t_0}^k=\mathbf{W}_{t_0}'$ for all $k \in \mathcal{K}$. 
Note that \eqref{ieq:bound_deviate} holds trivially for $t=t_0$.
For $t>t_0$, we have
\begin{align*}
\mathbb{E}\left[\sum_{k \in \mathcal{K}}p_k\Vert \mathbf{W}_{t}'-\mathbf{W}_{t}^k\Vert^2\right] &= \mathbb{E}\left[\sum_{k \in \mathcal{K}}p_k\Vert (\mathbf{W}_{t}^k-\mathbf{W}_{t_0}')-(\mathbf{W}_{t}'-\mathbf{W}_{t_0}')\Vert^2\right]\\    
&\stackrel{(a)}{\le} \mathbb{E}\left[\sum_{k \in \mathcal{K}}p_k\Vert \mathbf{W}_{t}^k-\mathbf{W}_{t_0}'\Vert^2\right] \\
&= \sum_{k \in \mathcal{K}}p_k\mathbb{E}\left[\Vert \mathbf{W}_{t}^k-\mathbf{W}_{t_0}'\Vert^2\right] \\
&\stackrel{(b)}{=} \sum_{k \in \mathcal{K}} p_{k} \mathbb{E}\left[\left\| \sum_{t'=t_0}^{t-1}\eta_{t'}\nabla F_{k}^{j_{t'}^{k}}(\mathbf{W}_{t'}^{k}) \right\|^2\right] \\
&\stackrel{(c)}{\le} \sum_{k \in \mathcal{K}} p_{k}  (\tau-1) \sum_{t'=t_0}^{t-1}\eta_{t'}^2\mathbb{E}\Big[\Vert \nabla F_{k}^{j_{t'}^{k}}(\mathbf{W}_{t'}^{k}) \Vert^2\Big]\\
&\stackrel{(d)}{\le} \sum_{k \in \mathcal{K}} p_{k}  (\tau-1)^2 \eta_{t_0}^2\theta_k^2\\
&\stackrel{(e)}{\le} 4(\tau-1)^2 \eta_{t}^2\sum_{k \in \mathcal{K}} p_{k} \theta_k^2
\end{align*}
where (a) holds since $\mathbb{E}[\Vert X-\mathbb{E}[X]\Vert^2] \le \mathbb{E}[\Vert X\Vert^2]$ where $X:=\mathbf{W}_{t}^k-\mathbf{W}_{t_0}'$, (b) holds since $\mathbf{E}_{t'}^k = \mathbf{0}$ for $t'=t_0,\ldots,t-1$ and iterates recursively over the first case of~\eqref{eq:local_sgd_alt}, (c) holds since $\Vert\sum_{t'=t_0}^{t-1}\mathbf{r}_{t'}\Vert^2 \le (t-1-t_0) \sum_{t'=t_0}^{t-1}\Vert\mathbf{r}_{t'}\Vert^2\le (\tau-1) \sum_{t'=t_0}^{t-1}\Vert\mathbf{r}_{t'}\Vert^2$, (d) holds by Assumption 2, and (e) holds due to $\eta_{t_0} \le \eta_{t-\tau}\le 2\eta_{t}$.
\fi

\bibliographystyle{IEEEtran}
\bibliography{ref}

\end{document}